%% file: main-sigmetrics-arxiv.tex
\documentclass[letterpaper,11pt]{article}
\pdfoutput=1 
\usepackage{amsmath}
\usepackage{amsthm}
\usepackage{amsfonts}
\newcommand{\BigO}[1]{\ensuremath{\operatorname{O}\bigl(#1\bigr)}}
\usepackage{algorithm}
\usepackage{algpseudocode}
\usepackage{comment}
\usepackage{amsthm}
\usepackage{algpseudocode}
\usepackage{amssymb}
\usepackage{algorithm}
\usepackage{graphicx}
\usepackage{epstopdf}
\DeclareGraphicsExtensions{.eps}
\usepackage[export]{adjustbox}
\usepackage{color}
\usepackage{float}
\usepackage{pgfplots}
\usepackage[normalem]{ulem}
\usepackage{fullpage}
\usepackage{algpseudocode}
\usepackage{authblk}
\newtheorem{theorem}{Theorem}
\newtheorem{lemma}[theorem]{Lemma}

\newtheorem{proposition}[theorem]{Proposition}
\newtheorem{corollary}[theorem]{Corollary}

\newtheorem{remark}[theorem]{Remark}

\date{}
\makeatletter
\newcommand\footnoteref[1]{\protected@xdef\@thefnmark{\ref{#1}}\@footnotemark}
\makeatother
\newlength\figureheight \newlength\figurewidth
\newcommand{\mytilde}{\raise.17ex\hbox{$\scriptstyle\mathtt{\sim}$}}
\title{Social Learning in Multi Agent Multi Armed Bandits}

\author{%
	Abishek Sankararaman \footnote{\noindent Electrical and Computer Engineering, The University of Texas at Austin. Email - abishek@utexas.edu},\hspace{4mm}
	Ayalvadi Ganesh \footnote{Department of Mathematics, University of Bristol. Email - A.Ganesh@bristol.ac.uk},\hspace{4mm}
	 Sanjay Shakkottai, \footnote{Electrical and Computer Engineering, The University of Texas at Austin.  Email - shakkott@austin.utexas.edu}
}

\begin{document}
	\maketitle

\begin{abstract}
	
  Motivated by emerging need of learning algorithms for large scale
  networked and decentralized systems, we introduce a distributed
  version of the classical stochastic Multi-Arm Bandit (MAB)
  problem. Our setting consists of a large number of agents $n$ that
  collaboratively and simultaneously solve the same instance of $K$
  armed MAB to minimize the average cumulative regret over all agents.
  The agents can communicate and collaborate among each other
  \emph{only} through a pairwise asynchronous gossip based protocol
  that exchange a limited number of bits. In our model, agents at each point decide on
  (i) which arm to play, (ii) whether to, and if so (iii) what and whom to
  communicate with. Agents in our model are decentralized, namely their actions only depend on their observed history in the past.
  
  We develop a novel algorithm in which agents, whenever they choose, communicate only arm-ids and not samples, with another agent chosen uniformly and independently at random. The per-agent regret scaling achieved by our algorithm is
  $\BigO{ \left( \frac{\lceil\frac{K}{n}\rceil+\log(n)}{\Delta}
      \log(T) + \frac{\log^3(n) \log \log(n)}{\Delta^2}
    \right)}$. Furthermore, any agent in our algorithm communicates (arm-ids to an uniformly and independently chosen agent)
  only a total of $\Theta(\log(T))$ times over a time interval of $T$.
  
  We compare our results to two benchmarks - one where there is no
  communication among agents and one corresponding to complete
  interaction, where an agent has access to the entire system history
  of arms played and rewards obtained of all agents. We show both
  theoretically and empirically, that our algorithm experiences a
  significant reduction both in per-agent regret when compared to the
  case when agents do not collaborate and each agent is playing the
  standard MAB problem (where regret would scale linearly in $K$), and
  in communication complexity when compared to the full interaction
  setting which requires $T$ communication attempts by an agent over
  $T$ arm pulls.

\end{abstract}

\maketitle


\input{introduction}

\section{Algorithm }
\label{sec:algo}

The algorithm has four parameters, $L, M, T_0 \in \mathbb{N}$ and $\alpha$, the UCB parameter. 
The algorithm evolves with the different agents being in different \emph{states} or \emph{phases} taking values in $\{-M,-M+1,\cdots, 0,1,\cdots\}$. At the beginning of execution, all agents start out in state $-M$, and as the execution proceeds, they increment their phase by $1$. In other words, the state of every agent is non-decreasing with time. We say that an agent is in \emph{Early Phase}, if its state is $-1$ or smaller, and in \emph{Late Phase} if its state is $0$ or larger. 

\subsection{Notation}
\label{sec:algo_notation}
For each agent $i \in \{1,\cdots,n\}$ and phase $j \in \{-M,\cdots\}$, we denote by $A_i^{(j)} \subseteq \{1,\cdots,K\}$ to be the set of arms agent $i$ is \emph{aware} of at the beginning of phase $j$. The algorithm is such that in any phase, an agent will only play from among the set of arms it is aware of.
{In our algorithm, every agent, if it chooses to communicate, will only communicate arm ids}. Thus, during the course of execution of our algorithm, agents will receive arm ids as messages. 
\\

For an agent $i \in [n]$ and phases $j, k \in \{-M,\cdots\}$, denote by $B_{i}^{(j)}\subset \{1,\cdots,K\}$, the set of arms \emph{received} by agent $i$, while agent $i$ is in phase $j$. At the start of phase $j \geq -M+1$, agent $i$ updates the set of arms it is aware of as $A_i^{(j)} = A_i^{(j-1)} \cup B_i^{(j-1)}$. In other words, agents update the set of arms they are aware of only at the end of a phase. Agents agree upon an initial set of arms, i.e., $A_i^{(-M)}$ is chosen before execution of the algorithm. Notice that the set of arms an agent is aware of is non-decreasing, i.e., if for any $j \in \{-M,\cdots,\}$, any agent $i \in \{1,\cdots,n\}$ and arm $l \in \{1,\cdots,k\}$, $l \in A_i^{(j)} \implies \forall h \geq j, l \in A_i^{(h)}$. 
\\

For any agent $i \in \{1,\cdots,n\}$, any arm $l \in A_i^{(j)}$, and any $k \in \mathbb{N}$, denote by $N_{l;i}^{(j)}(k)$ the number of times arm $l$ was played by agent $i$ during its first $k$ plays (epochs) in phase $j$. If $N_{l;i}^{(j)}(k) > 0$, denote by $\hat{\mu}_{l;i}^{(j)}(k)$ the empirical estimate of the mean of arm $l$ by agent $i$, \emph{using only the samples collected in the first $k$ plays of agent $i$ in state $j$}. 




\subsection{Algorithm Description} 
For any agent $i \in \{1,\cdots,n\}$, its execution is defined as follows. 
\\
%
%
%


\noindent\textbf{Initialization} - At time $t=0$ (i.e., at the beginning of phase $-M$),  agent $i \in \{1,\cdots,n\}$ is aware of arms $ A_i^{(-M)} =  \left\{ \left( i \lceil \frac{K}{n} \rceil \mod K\right) +1 ,\cdots,\left((i+1)\lceil \frac{K}{n} \rceil -1\mod K\right)\right\}$. Observe that the cardinality $|A_i^{(-M)}| = \lceil \frac{K}{n} \rceil$. In this initialization step, we assume agents are aware of their arm-ids for ease of exposition. In the sequel in Remark \ref{remark_random_init}, we give a randomized initialization where each agent in the beginning, is aware of a random set of arms chosen independently without knowledge of agent-ids. 
\\


%
%


\noindent \textbf{Early-Phase} - When agent $i$ is in any {phase} $j \in \{-M,\cdots,-1\}$, it plays from among the arms in $A_i^{(j)}$ in round-robin fashion. Agent $i$ is in any early phase $j$ for precisely $L$ times, i.e., for exactly $L$ epochs of its clock process $C_i(\cdot)$, before shifting to state $j+1$. At the end of ($L$ plays in) phase $j$, agent $i$ chooses another agent uniformly and independently at random, and communicates to it the index (id) of the arm from $A_i^{(j)}$ having the highest empirical mean based on the samples collected during phase $j$. 
\\
%


\noindent \textbf{Late-Phase} - Agent $i$ is in this late-phase, if its phase $j \in \{0,1,\cdots\}$. Agent $i$ is in phase $j \in \{0,1,\cdots\}$ for exactly $T_j$ epochs, where $T_j : = \lfloor \frac{T_0}{2} 2^{2^j} \rfloor -  \lfloor \frac{T_0}{2} 2^{2^{j-1}} \rfloor$, before shifting to phase $j+1$. At any play instant $k \in [T_j]$ of agent $i$ in phase $j$, if there is an arm $l \in A_i^{(j)}$ such that $N_{l;i}^{(j)}(k) = 0$, it plays one such arm, chosen arbitrarily. If no such arm exists, agent $i$ then plays an arm  chosen according to the UCB policy \cite{ucb_auer}, i.e., the arm is chosen from the set
\begin{align*}
\arg\max_{l \in A_i^{(j)}} \left( \hat{\mu}_{l;i}^{(j)}(k-1) + \sqrt{ \frac{ \alpha \log(k)}{N_{l;i}^{(j)}(k-1)}} \right).
\end{align*}
Furthermore, for all late phases $j \geq 1$, agent $i$ communicates only for the first $n2^j$ epochs and after that does not communicate. Agent $i$ will communicate in phase $j$, the arm from the previous phase $O_i^{(j)} \in A_i^{(j-1)}$ that was played the most number of times, with each communication attempt directed at an uniform random agent.

\subsection{Algorithm PseudoCode}

For ease of readability, we translate the above description of our algorithm into pseudo-code in Algorithm \ref{algo:main-algo}. This algorithm assumes access to a function called {\ttfamily Communicate}, that takes in an arm-id $\xi \in \{1,\cdots,K\}$ and an agent $y \in \{1,\cdots,n\}$ as input and sends arm-id $\xi$ to an agent chosen uniformly at random from $\{1,\cdots,n\}\setminus \{y\}$ and independently of everything else.


\begin{algorithm}[H]
	\caption{Distributed MAB Regret Minimization (at Agent $i$)}
	\begin{algorithmic}[1]
		\State \textbf{Input}: $M,L,T_0,\alpha$
		\State \textbf{Initialization}:  \\ $ A_i^{(-M)} = \left\{ \left( i \lceil \frac{K}{n} \rceil \mod K\right) +1 ,\cdots,\left((i+1)\lceil \frac{K}{n} \rceil -1\mod K\right)\right\}$
		\For{Epochs $t \in \mathbb{N}$ of clock process $C_i(\cdot)$}
		\If{ $t \leq ML$} \Comment{Early Phase}
		\State $j \gets t \mod L$ \Comment{Current Phase number}
		\If{$t \mod L == 0$} \Comment{End of a Phase}
		\State   {\ttfamily Communicate}($\arg\max_{l \in A_j^{(i)}} \hat{\mu}_l^{(i,j)}$,$i$) 
		\State $A_{i}^{(j+1)} \gets A_i^{(j)} \cup B_i^{(j)}$
		\Else 
		\State Play arm from $A_i^{(j)}$ in round-robin
		\EndIf

		\Else \Comment{Late Phase}
		\State $j \gets \inf \left\{m \geq 0 : t \leq ML + \lfloor \frac{T_0}{2} 2^{2^m} \rfloor  \right\}$
		\If{$t == ML + \lfloor \frac{T_0}{2} 2^{2^j} \rfloor $} \Comment{New Phase}
		\State $A_{i}^{(j+1)} \gets A_i^{(j)} \cup B_i^{(j)}$
		\EndIf
		\If{$\exists l \in A_i^{(j)}$ such that $N_{l;i}^{(j)} = 0$}
		\State Pull arm $l$
		\Else
		\State Pull - $\arg\max_{l \in A_i^{(j)}} \left( \hat{\mu}_{l;i}^{(j)}(t-1) + \sqrt{ \frac{ \alpha \log(k)}{N_{l;i}^{(j)}(t-1)}}\right)$
		\EndIf
		\If{$t- ML -\lfloor \frac{T_0}{2} 2^{2^j} \rfloor\leq n2^j$ {\ttfamily \textbf{AND}} $j \geq 1$}
		\State {\ttfamily Communicate}($\arg\max_{l \in A_i^{(j-1)}}  N_{l;i}^{j-1}(T_{j-1})$ ,$i$) \Comment{The arm most played in the previous phase ($j-1$)}
		\EndIf

		\EndIf
		\EndFor
	\end{algorithmic}
	\label{algo:main-algo}
\end{algorithm}

%


\subsection{Remarks on the Algorithm}


The algorithm is `fully asynchronous' in the sense that agents act independently without keeping track of either a absolute continuous time, or a shared global system clock.  Notice that in the early-stage, every agent communicates exactly $M$ times, which we will later set to be $\Theta(\log(n))$ in the sequel. In each late-stage $j \geq 1$, an agent communicates for exactly $n2^j$ times. Since the duration of each late-stage phase is doubly exponential, after $T$ time steps of play of any agent, it would have communicated order $\log(T)$ number of times, where each communication is of $\log_2(K)$ bits. 
\\

One can potentially improve the algorithm, by using a black-box best arm identification in the early-phase of an agent instead of playing arms in a round robin fashion. Concretely, if $\mathbb{A}$ is any best-arm identification algorithm, then in each early phase $j \in \{-M,\cdots,-1\}$, each agent $i \in \{1,\cdots,n\}$, will use the algorithm $\mathbb{A}$ on the set of arms $A_j^{(i)}$ for at-most $L$ total arm pulls. If at the end of $L$ arm pulls, either a best arm from $A_j^{(i)}$ is identified which is communicated, or the algorithm $\mathbb{A}$ fails to terminate within $L$ steps, in which case a random arm from $A_j^{(i)}$ will be communicated and agent $i$ moves to phase $j+1$. Similarly, one could use a more sophisticated version of the UCB algorithm (\cite{lattimore_book}) in the late phase and obtain slightly better results.

\section{Main Result}
\label{sec:main_result}

\begin{theorem}
	Consider a system with $n \geq 1$  agents and $K\geq2$ arms, with each agent running the above algorithm with parameters {\color{black}$M = \lceil  361 \log(n) \rceil+1 $}, $L = \bigg\lceil \frac{2M + \lceil \frac{K}{n}\rceil}{\varepsilon^2} (18 M)\log(100(2M + \lceil\frac{K}{n}\rceil)) \bigg\rceil$ and $T_0 = \lceil \frac{\max(K^{2},n) \log(\varepsilon^{-1})}{\varepsilon^2} \rceil $, UCB parameter $\alpha = 3$ and {\color{black}}, where $0 < \varepsilon \leq \Delta$. Then for every agent $I \in \{1,...,n\}$ and $\forall T \in \mathbb{N}$, the regret after agent $I$ has played for $T $ epochs is bounded by 
	\begin{multline}
	\mathbb{E}[R_I^{(T)}] \leq  \frac{4 \alpha}{\Delta} 4(\widehat{M}) \log(T-T_0) \mathbf{1}_{T > T_0}  + ML  +  8T_0\left(\frac{{\color{black}150}\log(n)}{n^3}\mathbf{1}_{n \geq 29} + \mathbf{1}_{n < 29} \right) + \\ 2\log_2\left(\log_2\left(\frac{2T}{T_0}\right)\right) \left(  \frac{4 \alpha}{\Delta} \log\left( \frac{T_0}{2}\right)+  \widehat{M} \left( 1 + \frac{\pi^2}{3}\right) \right)\mathbf{1}_{T \geq T_0},
	\label{eqn:known_delta_main_result}
	\end{multline}
	where $\widehat{M} = 2M+\lceil \frac{3K}{n^2} \rceil + \lceil \frac{K}{n} \rceil$. Moreover, in $T$ epochs of play, each agent communicates at-most a total of $M + n \log(T/T_0)\mathbf{1}_{T \geq T_0}$ times.
	\label{thm:known_delta}
\end{theorem}

To help parse the result, we consider the case of $K=n$ in the following remarks to understand how effectively our algorithm is leveraging the collaboration among agents. 

\begin{remark}
	In the case $K=n$ and $n > 29$, Theorem \ref{thm:known_delta} states that the expected regret of any agent $I$ after $T$ epochs is $\BigO{\frac{\log(n)}{\Delta}\log(T) + \frac{\log^3(n)}{\Delta^2}\log\log(n) }$. We can compare this regret scaling with the two benchmark systems of no communication and complete interaction described in Section \ref{sec:benchmark}. In case agents do not interact at all, the per-agent regret is known (\cite{lai_robbins}) to scale as $\BigO{( \frac{n}{\Delta}\log(T)}$. In the setting of complete information exchange however, from the discussion in Section \ref{sec:benchmark} adapted to the case $K=n$ yields that the per agent regret scales as $\BigO{ \frac{1}{\Delta}\log(T) + \frac{\log(n)}{n\Delta}}$. Thus, our algorithm is off by only by a logarithmic factor in $n$ with respect to full coordination plus an additive constant regret term of $\BigO{\frac{\log^3(n) \log \log(n)}{\Delta^2}}$. 
	\label{remark_regret}
\end{remark}

%
%

\begin{remark}
	Recall that in the fully centralized setting, the total number of times an agent communicates with the centralized server is $T$, if an agent plays for $T$ epochs. This follows as for each play of the agent, the centralized entity must communicate an arm-id for the agent to play which will require at-least $\log_2(n)$ bits and the agent reports back its observed samples which takes $1$ bit. However, in our algorithm, the total number of communications initiated by an agent in $T$ epochs is order $\log(T)$, where each communication is at-most $\log_2(n)$ (which is equal to $\log_2(K)$ in this example) bits, similar to the setting with complete information exchange.
	\label{remark_communication}
\end{remark}

Further, if the arm rewards are drawn from a more general sub-Gaussian distribution, the analysis in this paper will go through with minor modifications, and both the regret scaling and communication scaling remains unchanged. However, this relaxation has implications on the communication complexity with a centralized algorithm. Specifically, each agent needs to encode and communicate the arm reward at a sufficient resolution to distinguish between the best and next best arm mean, which will take an additional $\Theta(1/\Delta)$ bits (assuming $\Delta$ is known) per message.  
\\

Thus, our algorithm is able to effectively emulate the complete interaction setting using only pairwise anonymous asynchronous gossip-style communications with much smaller communication complexity. 

\begin{remark}
	We note that the constants in front of $M$ and $L$ is sub-optimal as it arose from certain tail probability bounds which are not tight. In all simulations in this paper, we set $M = \lceil3\log(n)\rceil+1$ and $L = \lceil 0.8 \frac{2M + \lceil \frac{K}{n}\rceil}{\varepsilon^2}\log(10(2M + \lceil\frac{K}{n}\rceil)\varepsilon) \rceil$. We see from our plots in Figures \ref{fig:regret_bound_intro}, \ref{fig:synthetic} and \ref{fig:real_data}, that this choice works well in practice.
\end{remark}

\begin{remark}
	The choice of the parameter $\varepsilon$, appears from Theorem \ref{thm:known_delta} to crucially affect the performance of our algorithm. However, we see numerically in Section \ref{sec:simulations}, that our algorithm enjoys good performance for a range of values of $\varepsilon$, even if $\varepsilon > \Delta$.
\end{remark}

{\color{black}


\begin{remark}
	The initialization in Line $3$ of Algorithm \ref{algo:main-algo} requires agents to be aware of their index, which may not be feasible in many scenarios. The following simple modification to Line $3$ can make our algorithm fully distributed.  Given any positive $\gamma \in (0,1)$, each agent $i \in \{1,\cdots,n\}$, will construct its initial set $A_i^{(-M)}$, by choosing $M_{\gamma}:=\bigg\lceil \frac{\ln\left(\frac{1}{\gamma}\right)}{n\ln \left( \frac{K}{K-1}\right)}\bigg\rceil$ arms from the set $\{1.\cdots,K\}$, uniformly at random with replacement. Then, with probability at-least $1-\gamma$, there will exist an agent $j \in \{1,\cdots,n\}$ such that $1 \in A_j^{(-M)}$, i.e., the best arm is in some agent's initial playing set. On this event, the regret of any agent $I \in \{1,\cdots,n\}$ after playing for $T$ time steps will satisfy 
	\begin{multline}
	\mathbb{E}[R_I^{(T)}] \leq  \frac{4 \alpha}{\Delta} 4(\widehat{M}+M_{\gamma}) \log(T-T_0) \mathbf{1}_{T > T_0} + ML  +  8T_0\left(\frac{{\color{black}150}\log(n)}{n^3}\mathbf{1}_{n \geq 29} + \mathbf{1}_{n < 29} \right) + \\ 2\log_2\left(\log_2\left(\frac{2T}{T_0}\right)\right) \left(  \frac{4 \alpha}{\Delta} \log\left( \frac{T_0}{2}\right)+  (\widehat{M}+M_{\gamma}) \left( 1 + \frac{\pi^2}{3}\right) \right)\mathbf{1}_{T \geq T_0}.
	\end{multline}
	All occurrences of $\widehat{M}$ in Equation (\ref{eqn:known_delta_main_result}) is replaced by $\widehat{M}+M_{\gamma}$.
	\label{remark_random_init}
\end{remark}
}

%

\subsection{Discussion}


The per-user regret bound in Equation (\ref{eqn:known_delta_main_result}) implies several objectives accomplished by the algorithm.
First, it establishes that every agent will play the best
arm eventually with probability $1$. For if an agent did not play the
best arm ever with some probability $\delta > 0$, then the per-user
regret has a lower bound of $ \frac{\delta}{n} T$, which for fixed
$n$, the scaling in time is not logarithmic. Second, since an agent
only chooses arms from the set of arms it is aware of, the regret
bound also implies that on average, a typical agent plays
at-most order $\log(n)$ number of arms. These two properties of
{\em (i)} every agent being aware of the best arm, while {\em (ii)}
playing a total of order $\log(n)$ number of distinct
arms illustrates the key benefit of collaborative messaging. In words,
{\em collaboration spreads the best arm to all other agents while not
	spreading the poor arms, so that not all agents need to learn and
	discard the poorly performing arms.}
\\


Furthermore, observe that our regret bound has an additive term that scales
as $\frac{\log^3(n)}{\Delta^{2}}$.  This additive term can be viewed as a \emph{cost of collaboration through the gossip noisy process}. As nodes only play from the set of arms they are aware, a node may not play the best arm until it is recommended and will keep incurring a regret linear with time. However, from well known results (\cite{lil_ucb}), it takes an agent at-least order ${\Delta^{-2}}$ epochs to identify the best arm with a constant probability and thus to communicate it through the gossip process. Thus, the term $\frac{\log^3(n)}{\Delta^{2}}$ is the average time before a typical agent is aware of the best arm and starts playing it. We refer the reader to Appendix \ref{sec:interpretation} for more discussion.


\subsection{Algorithm Intuition and Challenges in Analysis}
\label{sec:challenges}

%

Our goal is to design an algorithm so that all agents become aware of
the best arm as quickly as possible, since agents will incur a
linearly scaling regret until they become aware of the best arm.
Thus, we conceptually,  divide the evolution of the algorithm into two
stages: an {\em early stage} and a {\em late stage}.
\\

\noindent \textbf{The early stage:} In this stage, gossip and best arm
identification dominates, where the goal is to ensure all agents have
identified the best arm, but simultaneously making sure that each
agent is only aware of and has explored a small fraction of the arms.
The tension is the following: \emph{When not all agents are even aware of
	the best arm, agents must aggressively spread or communicate what they
	estimate as their current best arm. However, if agents communicate too
	frequently, then their estimates are likely to be poor, as they will
	be based on too few samples, thus leading to both increased
	communications and bad recommendations (resulting in all agents being
	aware of too many arms and leading to poor regret scaling)}. 
\\


\noindent \textbf{The late stage:} As time progresses when all agents
are reasonably sure of being aware of the best arm, agents must start
focusing on regret minimization rather than estimating best arms.
However, since we want to ensure that all agents are aware of the best
arm eventually with probability $1$, agents must nevertheless keep
communicating. In particular, almost-surely, {\em all agents must
	eventually make infinite recommendations as time progresses,
	while only making small and finitely many incorrect
	recommendations.} Thus, the late-stage must be designed to balance two competing objectives. $(i)$ In the rare case that not all agents are aware of the best arm when they shift to the late-stage, they must become aware of the best arm quickly and, $(ii)$ in the typical case when all agents are aware of the best arm at the beginning of the late-phase, the number of new arms an agent becomes aware of in the late-stage must be small. The second objective is desirable as all newly aware arms in the late-phase, conditioned on agents being aware of the best arm at the end of the early-phase will necessarily be sub-optimal arms. 
\\

\noindent \textbf{Recommendations} - In our algorithm, we \emph{decouple} the samples (reward of arm pulls) on which agents make successive recommendations, both in the early as well as late phases. This allows us to claim that the quality of recommendations by an agent are independent across phases, which aids greatly in the analysis. This decoupling also ensures that the {\em quality of recommendations made by agents be
	independent of the regret an agent obtains on its samples.}  We achieve this independence by using the
doubling trick \cite{doubling_trick} in the late-phase and using the performance of an agent in the phase before to make recommendations in the current phase. Contrary to the main
uses of the doubling method in converting a fixed horizon algorithm
into a anytime algorithm, we use this to provide the necessary sample
splitting, between making recommendations and minimizing regret. This decoupling comes at a price however which shows up as an additive $\Theta(\log(\log(T)))$  term in the regret.
\\

\noindent \textbf{The parameter $\varepsilon$ in our algorithm:}  Observe that the algorithm needs $0 < \varepsilon \leq \Delta$ for the regret guarantees to hold. Furthermore, the closer this parameter is to $\Delta$, the better is our regret bound, evidenced both by our Theorem \ref{thm:known_delta} and simulations in Section \ref{sec:simulations}. However, we show empirically in Section \ref{sec:simulations}, that even if $\varepsilon > \Delta$, in practice our algorithm yields good performance and leverages the benefit of collaboration.
\\

The knowledge of $\Delta$ is particularly helpful to
agents in deciding when to make recommendations, i.e., the choice of both $L$ and $T_0$. If an agent
recommends too early in the early stage, say much smaller than playing $\Delta^{-2}$
times in total, then such a recommendation will likely be wrong.
One potential method to remove requiring knowledge of $\Delta$ would
be for agents to run a fixed-confidence best arm identification
algorithm (e.g. see \cite{best_arm_survey} and references therein)
before making recommendations. However, such a modification to our
algorithm is not guaranteed to work. To see this, consider a problem
instance where $\mu_2 - \mu_{3} << \mu_1 - \mu_2 := \Delta$, with all
agents being aware of arms $2$ and $3$ in the beginning. In the early
phase of the algorithm when not all agents are aware of the best arm,
those that are not aware of the best arm (but have arms 2 and 3) will
spend a large number of samples in order to distinguish between these
two arms. Thereby, these agents will stay in the early phase for a
long time, thus incurring a large regret. However, as neither of these
are the best arm, it does not matter which of these two arms is
recommended, and hence agents could have used fewer samples and have saved
on incurring regret. 
\\

We remark here that this assumption seems to be made for many algorithms developed to leverage collaboration in a networked setting. As mentioned, the simple regret counterpart to our cumulative regret in a networked setting considered in \cite{hillel} and \cite{p2p_simple_regret} assume knowledge of $\Delta$ for their algorithms. For instance, algorithms of \cite{hillel} require $T$, an input parameter to be larger than a certain function of $\Delta$, while \cite{p2p_simple_regret} requires an explicit lower bound on $\Delta$ similar to ours.

\input{proof_sketch}
\input{simulations}

\section{Related Work}
\label{sec:related_work}

Our work focuses on multi-armed bandit (MAB)
\cite{mab_first, bubeck_surveyt}
problems in a {multi-agent setting,} which has received increasing
attention in a number of applications. The earliest
work in this direction is
\cite{kleinberg_collaborative_learning}, which consider an
adversarial bandit model with malicious
agents. This setting was
further developed in \cite{delay_non_stochastic}, with delays in communication among agents which were connected by a general graph. However,  there are no restrictions on the communications and
agents in these models could communicate after every arm-pull. Subsequently, \cite{kanade2012distributed},  studies the communication versus regret trade-off in a distributed setting with non-stochastic bandits. However, their model does not impose pairwise  communication, rather agents communicate  via a central coordinator. 
In the non-stochastic setting, \cite{prediction_limited_advice} introduces interactions across agents as limited advice from
experts and thus different from our setting.
\\

%
%

In the stochastic bandit setting, the papers \cite{stochastic_team}, \cite{buccapatnam} studies the trade-off between communication cost and regret minimization among a team of bandits. However, in these models, agents can simultaneously share information with all others and thus different from the pairwise communication setting of this paper. The model in \cite{kjg18} considered a multi-agent bandit optimization on a social network, where the action and reward of an agent can be observed by neighbors on a graph. However, there is no notion of communications versus regret trade-off as agents communicate to their neighbors at all time steps in their model.  A recent work of \cite{kanade_stochastic} considered a multi agent setup where agents can choose to communicate with all neighbors on an underlying unknown graph. However, agents in their algorithm communicate after each arm-pull and thus do not have a communications versus regret trade-off.
\\

There has also been work (\cite{hillel},\cite{p2p_simple_regret}) in understanding the communication versus simple regret (pure explore) trade-off for best arm identification, which is different from the cumulative regret (explore-exploit trade-off) considered in this paper. Moreover, information sharing in these models are different from ours - the communication model of \cite{hillel} is one where every node can see every other node's message, whereas the agents in \cite{p2p_simple_regret} can communicate at each time step and hence the communications per agent is linear in the number of arm-pulls. However, similar to our paper, both these papers require some knowledge of the arm-gap $\Delta$. The algorithm of \cite{hillel} is guaranteed to work if the time horizon $T$, which is an input parameter, exceeds a function of $\Delta$, while the algorithm in \cite{p2p_simple_regret} requires an explicit lower bound on $\Delta$. 
\\


The paper
\cite{distributed_consensus_plus_solve} considers a distributed
bandit setting where agents communicate arm means using
a consensus algorithm without any communication limitations, unlike our setting. There has also been a line of work
(\cite{distributed_no_com}, \cite{musical_chair}, \cite{avner_cognitive},
\cite{decentralized_multi_agent_collision},
\cite{distributed_no_communication}, \cite{anima_distributed}) where the agents are competitive, unlike our setting, and interact only indirectly by observing each others' rewards. The paper of \cite{network_regret} considers a  model with different arm means for agents. In each time-step, a single action is taken by the network as a whole through voting process unlike ours where each agent takes an action. The paper \cite{gang_of_bandits} considers a single centralized learner that is playing multiple contextual bandit instances, where each instance corresponds to a user on a graph. The graph encodes interactions where `nearby users' on the graph have `similar' contextual bandit instances,  different from interactions in our model. Recent works  \cite{zubeida},
\cite{social_learning_replicator} have considered the social learning problem where agents do best-arm identification (simple regret). In these
setups, the memory of an agent is limited, and hence standard bandit algorithms such as UCB is infeasible. Rather agents resort to simpler algorithms such as the replicator dynamics and thus,
their algorithmic paradigm is not applicable to our setting. 
\\


Developments in large scale distributed computing is prompting the study of other learning questions in a decentralized setting.
For instance \cite{duchi_distributed}, \cite{distributed_mirror}, \cite{nedic_distributed},
\cite{boyd_distributed}, \cite{mass_optimal}, \cite{shi_extra}, study
multi-agent convex optimization with gossip style communications. More classically, gossip based computation models has a rich line of history under
the name of population protocols \cite{population_protocols} and rumor
spreading (\cite{epidemic_database}, \cite{karp_rumor_spread}). We refer the reader to
\cite{shah_gossip} and related references for other applications of the gossip mechanism.

\section{Conclusion and Open Problems}


In this paper, we study a problem of collaborative learning when there are a
group of agents playing the same instance of the MAB problem. We
demonstrate that even with limited collaboration, the per agent regret
is much smaller when compared to the case when agents do not
collaborate. The paper however motivates several open questions. An immediate
question is how to design an algorithm in which the agents are not
aware of the arm-gap $\Delta$. This is particularly challenging since
an agent is not aware of when to make recommendations, i.e., agents must balance both best-arm identification as well minimizing simple regret. Even the state of art, best-arm identification algorithms in a networked setting also needs knowledge of $\Delta$ (\cite{hillel},\cite{p2p_simple_regret}).
 Another question that arises from our work is to understand other algorithmic paradigms to
exploit collaboration. In this paper, we considered the scenario where
agents only play from among the arms it is aware of, where
collaboration is key to expanding the set of arms an agent is aware
of. Are there natural protocols, where the set of arms an agent is
aware of can be modeled in a `soft' fashion, where agents
\emph{prefer} to play those arms that has been recommended to it more
than other arms that have been recommended fewer number of times. This
is a challenging problem, both from an algorithmic design perspective
and also from a mathematical stand point. Third, can Theorem
\ref{thm:discrete_rumor_2} be tightened to get precise limiting
theorems similar to those obtained in \cite{discrete_rumor} and
\cite{pittel}. Such a result will help in reducing the constants in the
definition of $M$ and $L$.
\\

\noindent{\bf Acknowledgements} - This work is partially supported by NSF Grant CNS-1704778, ARO grant W911NF-17-1-0359 and the US DoT supported D-STOP Tier 1 University Transportation Center. AS acknowledges several stimulating discussions on the model with Rajat Sen, Soumya Basu and Karthik Abinav Sankararaman. AS also thanks Fran\c cois Baccelli for the support and generous funding through the Simons Foundation Grant ($\#$197892 to The University of Texas at Austin). 

\bibliographystyle{plain}
\bibliography{main_sigmetrics_arxiv}

\appendix

\input{analysis_abishek_scratch_changed}

\section{Early Stage Analysis - Proof of Lemma \ref{lem:early_good}}
\label{sec:proof_early_lemma}
%

The lemma states that
$$
\mathbb{P} (\mathcal{E}_1 \cap \mathcal{E}_2 \cap \mathcal{E}_3 ) \geq 1-150\log(n)n^{-3}.
$$
Thus, it suffices to prove that $\sum_{i=1}^3 \mathbb{P} (\mathcal{E}_i^c) \le 150 \log(n)n^{-3}$. This follows from Lemmas \ref{lem:early_e1}, \ref{lem:early_e2} and \ref{lem:early_e3} below. We first analyze events $\mathcal{E}_2^c$ and $\mathcal{E}_3^c$, for which the required inequalities follow directly from Chernoff tail bounds.

\begin{lemma}
	\begin{align*}
	\mathbb{P}[\mathcal{E}_2] \geq 1 - n^{-5}
	\end{align*}
	\label{lem:early_e2}
\end{lemma}
\begin{proof}
	Observe that for any agent $i$, the number of recommendations it receives from all the other $n-1$ agents, through all $M$ of their early phases, is a Binomial random variable with parameters $(n-1)M$ and $1/(n-1)$. This is because there are a total of $(n-1)M$ possible recommendations that can be made by the other agents during their early phases, and each of those recommendations reach agent $i$ with probability $1/(n-1)$, independent of everything else. Thus, from standard tail bounds, we have 
	\begin{align*}
	\mathbb{P}\left[\text{Bin}\left((n-1)M, \frac{1}{n-1}\right) > 2M-2 \right] & \leq \exp \left( -\frac{2}{3}M\right) \\ & \leq n^{-6}
	\end{align*}
	Thus, from an union bound, we observe that with probability at-least $1-n^{-5}$, every agent will receive lesser than or equal to $2M-2$ recommendations from another agent in phase $-1$.
\end{proof}

\begin{lemma}
	\begin{align*}
	\mathbb{P}[\mathcal{E}_3] \geq 1 - n^{-3}.
	\end{align*}
	\label{lem:early_e3}
\end{lemma}
\begin{proof}
	Any agent $i \in [n]$, will be in phase $M$ or larger at time $\mathcal{T}_n$, if in the time interval $[0,\mathcal{T}_n]$, at-least $ML$ clock ticks of the clock process $C_i(\cdot)$ has occurred. We can bound the probability of this not happening by a standard Chernoff bound as 
	
	\begin{align*}
	{\color{black}	\mathbb{P}\left[ \text{Poi}((M-1)L(1+\delta)) \geq ML  \right] \leq n^{-4},}
	\end{align*}
	since $\delta = \frac{1}{3M}$. Thus, by an union bound, at time $\mathcal{T}_n$, with probability at-least $1-n^{-3}$, no agent is in phase $0$ or larger.
\end{proof}

However, we remark that in the sequel in Lemma \ref{lem:discrete}, we shall prove a more stronger statement which implies Lemma \ref{lem:early_e3}. In order to bound the probability of event $\mathcal{E}_1$, we need the following result on bandit arm estimation, whose proof can be found in \cite{bubeck_ucb_simple_regret}. 
\begin{lemma}
	If an agent has $2M$ arms, where at each play instant it chooses an arm in a round robin fashion, then the probability that the arm corresponding to the highest empirical mean does not equal the arm with the highest mean reward after playing for\\ $L \geq \lceil \frac{2M}{\Delta^2} \log(200M\Delta)\rceil$ times is at-most $1/100$. 
	\label{lem:bandit_result_perr}
\end{lemma}

In order to bound the error probability $\mathcal{E}_1$, we consider a fictitious \emph{virtual system} and show in the sequel that with probability at-least $1-2n^{-3}$, the evolution of the virtual system coincides with that of our algorithm in the time interval $[0,\mathcal{T}_n]$. The virtual system, also consists of $n$ agents, with each agent playing the same bandit problem with the same $K$ arms. The algorithm employed by the agents in this virtual system is identical to the algorithm employed by the agents with identical initialization of arms in the early-phase with the following three additional modifications.
\begin{itemize}
	\item The agents are always in the early-stage till time $\mathcal{T}_n$. In particular, if any agent makes $M$ recommendations before time $\mathcal{T}_n$, it will continue to play further with the same early-stage protocol until time $\mathcal{T}_n$ in the virtual system.
	\item At time $\mathcal{T}_n$, the virtual system stops and no more activity occurs.
	\item At the beginning of any stage, if any agent in the virtual system has $2M+\lceil \frac{K}{n}\rceil +1$ or more arms under consideration for the next state, then it will drop some arms to ensure that it has exactly $2M+\lceil \frac{K}{n}\rceil$ arms to play at the beginning of the next state. The arm dropping policy is as follows. If an agent has $\mathcal{M} \geq 2M+\lceil \frac{K}{n} \rceil+1$ arms at the beginning of a state and the arm indexed $1$ (i.e., the best arm) is not in the agent's playing set, then the agent chooses a set of $2M+\lceil \frac{K}{n}\rceil$ arms from among the $\mathcal{M} $ arms it has uniformly and independently at random. If on the other hand, amongst the $\mathcal{M}$ arms that an agent has, the best arm, i.e., arm indexed $1$ is in the agent's bag, then the agent chooses a uniformly at random subset of arms of size $2M-1 + \lceil \frac{K}{n} \rceil$ from amongst the set of $\mathcal{M}-1$ arms it has. In other words, if an agent has the best arm in its set, it never gets dropped.
	
\end{itemize}

\begin{lemma}
	With probability at-least $1 - 2n^{-3}$, the above virtual system and the algorithm has identical sample paths uptil time $\mathcal{T}_n$.
	\label{lem:coupling}
\end{lemma}
\begin{proof}
	We construct a coupling of the virtual system and the algorithm through the same clock process $(C_i(\cdot))_{i=1}^{n}$ and the randomness for both sampling the  rewards of arms  and for the gossiping  communication process. Thus, on the event $\mathcal{E}_2$ and $\mathcal{E}_3$, this coupling construction produces identical sample paths in the virtual system and the algorithm. For on the event $\mathcal{E}_2$, no agent in the virtual system will `drop arms' and on the event $\mathcal{E}_3$, all agents in the original algorithm are in their early-phase. Lemmas \ref{lem:early_e2} and \ref{lem:early_e3} then give that with probability at-least $1-2n^{-2}$, both events $\mathcal{E}_2$ and $\mathcal{E}_3$ occurs.
\end{proof}

We now analyze the behaviour of this virtual system as it is somewhat easier, and then use the above coupling result to conclude about the algorithm in the early stage. Denote by the event $\tilde{\mathcal{E}}_1$ to be the event at time $\mathcal{T}_n$, all agents in the virtual system are aware of the best arm. 
\begin{lemma}
	\begin{align*}
	\mathbb{P}[\tilde{\mathcal{E}}_1^{c}] \leq 146 \log(n) n^{-3}.
	\end{align*}
	\label{lem:coupling_computation}
\end{lemma}

Before giving the proof of this Lemma, we notice that this immediately yields that in the original system:
\begin{lemma}
	\begin{align*}
	\mathbb{P}[\mathcal{E}_1^{c}] \leq 148\log(n)n^{-3}.
	\end{align*}
	\label{lem:early_e1}
\end{lemma}
\begin{proof}
	
	Denote by the random time $\hat{Y}_n$ to be the first time when all $n$ agents in the virtual system are aware of the best arm. In the event that by time $\mathcal{T}_n$, not all agents are aware of the best arm, i.e., on the event $\tilde{\mathcal{E}}_1^{c}$, let $\hat{Y}_n = \infty$. Similarly denote by the random time $Y_n$ to be the first time when all agents in the algorithm are aware of the best arm. Notice from the construction of the virtual system, that on the event the algorithm and the virtual system couples, we have $Y_n \geq \hat{Y}_n$, with $Y_n = \hat{Y}_n$ if both the algorithm and the virtual system couples and event $\tilde{\mathcal{E}}_1$ occurs. Thus, we have from total probability, the following chain
	\begin{align*}
	\mathbb{P}[\mathcal{E}_1^{c}] &= \mathbb{P}[\mathcal{E}_1^{c}, \text{Coupling occurs}] + \mathbb{P}[\mathcal{E}_1^{c}, \text{Coupling fails}] ,\\
	&\stackrel{(a)}{\leq} \mathbb{P}[\hat{Y}_n > \mathcal{T}_n, \text{Coupling occurs}] + \mathbb{P}[\text{Coupling Fails}],\\
	&\leq   \mathbb{P}[\hat{Y}_n > \mathcal{T}_n] + \mathbb{P}[\text{Coupling Fails}], \\
	&\stackrel{(b)}{\leq} 146\log(n)n^{-3} + 2n^{-3}, \\
	&\leq 148\log(n)n^{-3}.
	\end{align*}
	In step $(a)$, we use the fact that on the event that coupling occurs, we have $Y_n \geq \hat{Y}_n$, which is the same as event $\tilde{\mathcal{E}}_1^{c}$. In step $(b)$, we use the estimates from Lemmas  \ref{lem:coupling} and \ref{lem:coupling_computation}.
\end{proof}

We now return to Proof of Lemma \ref{lem:coupling_computation}. The key idea we will employ to prove Lemma  \ref{lem:coupling_computation} is to notice that w.h.p., the spreading dynamics in the virtual system behaves almost identical to that of a discrete rumor mongering system. Precisely, we can ensure that w.h.p., the recommendations made by the agents `follow in sync', i.e., no agent will make its $j+1$st recommendation, before all other agents finish making their $j$th recommendation, for all $j \in [0,M]$. 

\begin{lemma}
	With probability at-least $1-n^{-3}$, for all $j \in [1,M]$, for all agents $i \in [n]$, agent $i$ makes its $j$th early stage transition in the time interval $[ jL(1-\delta), jL(1 +  \delta)]$.
	\label{lem:discrete}
\end{lemma}
\begin{proof}
	Observe that it suffices to prove that both $\mathbb{P}[\text{Poi}(L(1-\delta)) > L] \leq n^{-4}$ and $\mathbb{P}[\text{Poi}(L(1+\delta)) < L] \leq n^{-4}$. If we establish this, then it follows from an union bound over all agents and all phases, the claim of the lemma holds. 	From elementary Chernoff bounds, it follows that $\mathbb{P}[\text{Poi}(L(1+\delta)) < L] \leq e^{- L \left(\delta + \ln \left( \frac{1}{1+\delta}\right)\right) } \leq  e^{-\frac{1}{2}L \delta^2}\leq n^{-5}$. Similarly, computing $\mathbb{P}[\text{Poi}(L(1-\delta)) > L] \leq e^{-L \left( \ln \left( \frac{1}{1-\delta}\right) - \delta\right)} \leq e^{-\frac{1}{2}L \delta^2} \leq n^{-5}$. 
\end{proof}

Since $\delta = \frac{1}{3M}$, we have for all $i \in [0,M-1]$, $iL(1+\delta) < (i+1)L(1-\delta)$. Thus, the above lemma gives us that all the arm recommendation events are `separated'. See also Figure \ref{fig:early_stage_discrete}. Thus, in light of Lemmas \ref{lem:discrete} and \ref{lem:bandit_result_perr}, we can consider the following discrete time system, which can be viewed as a `noisy spreading' version of the classical \cite{discrete_rumor} process. There are $n$ nodes, with nodes numbered $1$ initially possessing a message. In each time-step, every agent that has had the message for at-least $1$ or more time-steps, calls another agent chosen uniformly and independently at random and attempts to communicate the message. Each communication attempt is correct with probability at-least $99/100$, independent of everything else. More formally, for every agent $i \in [n]$, denote by $Y_i \in \mathbb{N}$ to be the first time agent $i$ learns of the rumor. By definition, $Y_1 = -1$, as initially, we assume that agent $1$ is aware of the message/rumor. The rumor spreads, where in each time $t \in \mathbb{N}$, all agents $i \in [n]$, such that $Y_i \leq t-2$, will attempt to communicate the rumor to another agent chosen uniformly and independently at random. Each communication attempt is successful with probability $p \in (0,1]$, independent of everything else. Denote by $S_n^{(p)} = \max_{i \in [n]} Y_i$, the first time when all agents are aware of the rumor.  Observe that this process differs from the classical rumor mongering process of \cite{discrete_rumor} in two aspects. First, not all agents that receive the rumor spreads it. Only those that have had the rumor for at-least $1$ time slot participate in spreading it. Second, each communication attempt is successful with probability $99/100$, rather than being deterministically successful. Nevertheless, we will show in Theorem \ref{thm:rumor_spread}, that this process behaves similar to the classical rumor mongering process of \cite{discrete_rumor}, i.e.,  $S_n = \BigO{\log(n)}$, with high probability. A precise statement is available in Theorem \ref{thm:rumor_spread} where the proof follows similar arguments as used in \cite{discrete_rumor}, which we produce here for completeness. The main reason we introduce this process as there is a natural coupling between the virtual system and the above described discrete process, which is summarized in Proposition \ref{prop:stoch_domination}.

\begin{theorem}
	For the noisy delayed rumor mongering process, we have for any $\gamma > 0$, $p \in (0,1]$ and all $n$ sufficiently large, 
	\begin{align*}
	\mathbb{P}[S_n^{(p)} \geq 2C(\gamma,p) \log (n)] \leq (2+ \log_{2-\eta}(n))n^{-(\gamma+1)},
	\end{align*}
	where {\color{black} $C(\gamma,p)$ and $\eta$ are given in Theorem \ref{thm:discrete_rumor_2}}.
	\label{thm:rumor_spread}
\end{theorem}

The proof of this is deferred to the Appendix \ref{sec:proof_rumor}. From Remark \ref{remark:constant_C} in the sequel, the above statement reads that for all $n \geq 29$, 
\begin{align}
\mathbb{P}[S_n^{(0.99)} \geq 361 \log(n)] \leq 145 \log(n) n^{-3}.
\label{eqn:numeric_gurantee}
\end{align}
Observe that the total number of early phases in our algorithm is equal to $C(2,0.99)  \log(n)+1$, where $C(2,0.99) \leq 361$ is from the above Theorem \ref{thm:rumor_spread}. Denote by the event $\mathcal{E}_d :=  \{S_n < 2C(2,0.99) \\ \log(n)\}$ in the above rumor mongering process.
\begin{proposition}
	There is a coupling between the virtual system and the discrete time rumor mongering process described above with noise probability $p=99/100$, such that on the event $\mathcal{E}_d$ in the rumor mongering process and the event in Lemma \ref{lem:discrete} in the virtual system, all agents $i \in [n]$ in the virtual system are aware of the best arm at time $Y_i L(1+\delta)$.
	\label{prop:stoch_domination}
\end{proposition} 
\begin{proof}
	Observe that the number of recommendations made by any agent in the virtual system is $C(2,0.99)\log(n)$. Thus, on the event $\mathcal{E}_d$, we have for all $i \in [n]$, $Y_i \leq M$.
	To describe the coupling, we map the rumor in the rumor mongering process to the best arm id in our algorithm. The success probability in the rumor spreading corresponds to the fact that agents in the algorithm recommend the best-arm. Since agents do not reuse samples, the independence of communications in the rumor mongering process follows from that in the algorithm. To conclude the proof, it suffices now to argue that a deterministic one step delay in the discrete rumor spreading process provides an upper bound to the process induced by our algorithm. Notice that if an agent receives the best arm in time slot $[iL(1-\delta),iL(1+\delta)]$ for some $i \in [M]$ (this will be the case under the event in Lemma \ref{lem:discrete}), then the agent that receives this arm, may already have shifted to the next stage, and in particular, will not recommend this received arm in the time interval $[(i+1)L(1-\delta),(i+1)L(1+\delta)]$. Thus, the deterministic one step delay in the discrete rumor mongering process provides an upper bound on the times when an agent starts considering the best arm for recommendation. 
\end{proof}

We are now ready to conclude the proof of Lemma \ref{lem:coupling_computation}.
\begin{proof} of Lemma \ref{lem:coupling_computation}.\\
	Notice that $M = 361\log(n)+1$ and from Theorem \ref{thm:rumor_spread} and Equation (\ref{eqn:numeric_gurantee}), we have $2C(2,0.99)  \leq 361$. Thus, from Proposition \ref{prop:stoch_domination}, and the estimates in Lemma \ref{lem:discrete}, we know that with probability at-least $1 - n^{-3}$, the virtual system is such that, the arm spreading process is dominated by the discrete rumor mongering process. Further, from Theorem \ref{thm:rumor_spread} and Equation (\ref{eqn:numeric_gurantee}), the rumor mongering process communicates the best arm-id to all agents before $M$ time slots with probability at-least $1-145\log(n)n^{-3}$. Thus, we have $\mathbb{P}[\tilde{\mathcal{E}}_1] \geq 1 - 146\log(n)n^{-3}$.
\end{proof}

\input{proof_rumor}


\section{Interpretation of the Result}
\label{sec:interpretation}

Our main result on the per-agent regret given in Equation (\ref{eqn:known_delta_main_result}) is a sum of $4$ terms, each of which has a natural interpretation.
\begin{enumerate}
	\item The term $ \frac{4 \alpha}{\Delta} 4\widehat{M} \log(T-T_0)$ is the usual logarithmic scaling with time of the UCB algorithm. However, this term states that on average, an agent is aware of no more than order $\log(n)+\lceil\frac{K}{n} \rceil$ arms, as $M = \Theta(\log(n))$ and thus a typical agent only explores order $\log(n)+\lceil\frac{K}{n} \rceil)$ arms on average. Nevertheless, all agents are aware of the best arm, which allows a logarithmic scaling of regret with time as opposed to a linearly scaling, which would be the case if an agent, with positive probability, is not even aware of the best arm eventually.
	\item The term $ML $ constitutes the regret an agent pays in the early-phase. Notice that every agent is in the early phase for exactly $ML$ epochs. In the early phase, agents are only involved in best-arm identification and gossiping and hence incur a linear regret. 
	\item The 
	$2\log_2\left(\log_2\left(\frac{2T}{T_0}\right)\right) \left(  \frac{4 \alpha}{\Delta} \log\left( \frac{T_0}{2}\right)+  \widehat{M} \left( 1 + \frac{\pi^2}{3}\right) \right)\mathbf{1}_{T \geq T_0}$ term is the regret incurred by agents in the late-phase due to not re-using samples across phases. Recall that, even in the late-phase, agents only play arms based on the observed arm rewards in the current phase and not based on the observed rewards of previous phases. This was done so as to ensure statistical independence between quality of late stage recommendations and observed regret. This however incurs a cost in the regret given by the term that scales as $\Theta (\log(n) \log\log(T))$. 
	\item The term  $8T_0\left(\frac{{\color{black}150}\log(n)}{n^3}\mathbf{1}_{n \geq 29} + \mathbf{1}_{n < 29} \right)$ accounts for the errors as agents can in rare cases, shift into late-phase without necessarily being aware of the best arm. This is term accounts for the regret incurred in the late-stage in case of agent $I$ not being aware of the best-arm at the beginning of phase $0$, and must wait for a certain duration before becoming aware of and paying the best arm.
\end{enumerate}

\section{Useful Tail Bounds}

In this section, we collect all useful Chernoff tail bounds for the various distributions for ready reference.

\begin{lemma}
	Let $X$ be a Poisson random variable of mean $\lambda > 0$. Then, for any $t > 0$, $\mathbb{P}[X > \lambda + t] \leq e^{-\frac{t^2}{2\lambda}h\left(\frac{t}{\lambda}\right)}$, and for any $0 <  t < \lambda$, $\mathbb{P}[X < \lambda - t] \leq e^{-\frac{t^2}{2\lambda}h\left(-\frac{t}{\lambda}\right)}$, where the function $h(u):=2 \frac{(1+u)\ln(1+u) - u}{u^2}$.
\end{lemma}

\begin{lemma}
	Let $X$ be a Binomial random variable, i.e., $X \sim \textrm{Bin}(n,p)$ for some $n \in \mathbb{N}$ and $p \in (0,1)$. Then, for any $\delta >0$, $\mathbb{P}[X > (1+\delta)np] \leq e^{-\frac{\delta^2}{2+\delta}np}$.
\end{lemma}

\begin{lemma}
	Let $X$ be an exponential random variable of mean $\lambda$ (i.e., parameter  $\frac{1}{\lambda}$). Then for any $t \geq 0$ $\mathbb{P}[X > t \lambda] \leq e^{-t}$.
\end{lemma}



\end{document}

%% file: introduction.tex
\section{Introduction}
\label{sec:intro}

The Multi Armed Bandit (MAB) problem is a fundamental theoretical model to study online learning and the exploration-exploitation trade offs associated with them. In this paper, we introduce a collaborative multi-agent version of the classical MAB problem which features a large number of agents playing the same instance of a MAB problem. Our work is motivated by the increasing need to design learning algorithms for several large scale networked systems. Some common examples include {\em (i)} social and peer-to-peer recommendation services catering to large number of users who are in turn connected by a network (\cite{collab_filter_bandits},\cite{distributed_social},\cite{collaborative_social_reco}, \cite{p2precommed}), {\em (ii)} a collection of distributed sensors or Internet of Things (IoT) devices learning about the underlying environment (such as road traffic conditions) and connected with each other through some communication infrastructure such as the wireless spectrum (\cite{spectrum_competetion,ml_wireless_nets}), and {\em (iii)} online marketplaces with many services catering to the same customer base, where the different services can potentially share data about the users in some privacy compatible form (\cite{info_share_market}) and learning in groups (\cite{word_of_mouth}, \cite{human_collective}).
\\


A common theme in many of
these applications is the presence of a single MAB instance,  which many agents are simultaneously playing to minimize their own cumulative regret. Importantly, the agents can collaborate to speed up learning by
interacting with each other \emph{only in some restricted form}. As an example, the number of bits communicated or the frequency of interactions among agents may be limited in settings where either the agents are geographically distributed and communications are expensive or when in a IoT network where the devices performing learning are energy constrained. Our objective in this paper is to understand the benefit of collaboration in speeding up learning under natural communication constraints.

\subsection{Model overview}

Our setting consists of a large number of agents $n$, that collaboratively solve the same instance of a stochastic $K$-armed MAB problem (\cite{bubeck_surveyt}), where each arm yields a binary valued reward. The objective of each agent is to take actions to minimize their own cumulative regret.
 If there were just one agent, or if the agents were oblivious to each other and did not collaborate, then each agent is playing independently, the classical $K$ armed MAB problem. In our model, the agents can potentially collaborate with each other in solving the MAB problem by sending messages to each other over a communication network connecting them. 
 \\
 





Formally, agents are equipped with independent Poisson clocks
(asynchronous system), and when an agent's clock rings, an agent takes
an `action'.
Each action of an agent consists of (i) which arm to play to observe a
reward, (ii) whether to communicate, and if so, (iii) what and with whom to
communicate. Our model imposes three constraints on the communications among the agents. Firstly, each agent, whenever it chooses to communicate, can do so with only one other agent and is thus, is  `gossip style communications'. Secondly, agents can only communicate a limited $\BigO{\log(nK)}$ bits, each time they choose to communicate. The number of bits communicated in each communication attempt cannot scale with time or depend on the problem instance. In particular, this forbids agents from sharing, either all their sample history or estimates of arm-means up to arbitrary level of precision. Thirdly, each agent can access the communication medium only $o(T)$ times over any horizon of $T$ pulls of arms. This restriction disallows agents from communicating each time they pull an arm and observe a reward. Thus, agents must aggregate their observed history in some form, where the size of the message does not increase with time and communicate. 
\\

The agents are \emph{decentralized} - namely the actions of each agent (which arm to play, whether to communicate and if so to whom and what to communicate), can only depend on the agents past history of arms played, rewards obtained and messages received.


\subsection{Model Motivations}
\label{sec:motivations}

We highlight two instances of our model to motivate our choice of
problem formulation and our restrictions on the communications among
agents.
\\

The first example is the setting of multiple users (aka agents) on a
social network, visiting restaurants in a city. In this case, the
restaurants can be modeled as arms in a MAB providing stochastic
feedback on its quality during each visit. Each visit by an agent to a
restaurant provides a (noisy) score, using which an individual agent
can update her/his opinions of restaurants. Furthermore, the social
network platform enables users or agents to personally communicate to
one another to exchange their experiences. The feedback constraint on
the number of bits translates to only recommending a restaurant
identity, as opposed to the real-valued score for that (and/or any
other) restaurants. If the agent communicates her/his top-scoring
restaurant (as is the case in our algorithm later), then agents are
implicitly sharing rankings (their current
top-choice) instead of scores, which is well-known to be more
interpretable (different people's scores are hard to compare). Our
framework thus provides a guideline to understand good policies for
the users to explore the city that efficiently leverage the
information exchanged on the underlying social network.
\\


A second example is from robotics, where several robot agents can
communicate over a wireless ad-hoc network in a cooperative foraging
task \cite{foraging-2004}. The robots need to forage for a high-reward
site from among several possible physically separated locations (these
sites constitute the arms of the bandit). Since the communciation
network is bandwidth constrained, and the robot agents can only
communicate (typically pair-wise) with those within their radio-range,
the communication constraints we consider are appropriate in this
setting. We also refer to \cite{kjg18} for another related robotics
example involving collaborative leak detection in a pipe system.

\subsection{Main Result}
\label{sec:summary}


We consider a setting with $n$ agents and $K$ arms of a MAB problem. The main result in this paper
is that we develop an algorithm that leverages collaboration across
agents, such that
the per-agent regret after an agent has played for $T$ times
scales\footnote{All logarithms in this paper are natural logs unless
	otherwise specified.} as
$\Theta \left( \frac{\log(n) + \lceil \frac{K}{n} \rceil }{\Delta}\log(T) + 
\frac{\log^3(n) \log\log(n)}{\Delta^2} \right)$, where $\Delta$ is the arm-gap
between the best and the second best arm. Moreover, in a time interval
of $T$, an agent communicates for about $\log(T)$ times, where each
communication is an arm-id, i.e., uses at-most $\log_2(K)+1$ bits per
communication. 
\\

The main idea in our algorithm is to use the communication medium only to \emph{recommend arms}, rather than to exchange observed scores or rewards. Our policy restricts agents to only play from the set of arms they are \emph{aware} of at any instant of time. Each agent is only aware of a small set of arms in the beginning,  and this set increases with time as agents receive recommendations. Agents in our algorithm communicate with another agent chosen uniformly and independently at random, and thus the communications induced by our algorithms is `gossip style' \cite{shah_gossip}.  Qualitatively, our regret scaling occurs due to two
reasons: {\em (i)} The (local-explore + gossip) mechanism underlying
our algorithm ensures that the {\em best arm spreads quickly through
	the network} to all agents. Notice that since agents only play from
among arms they are aware of, it is not apriori clear that all agents
become aware of the best-arm at all.  {\em (ii)} Nevertheless, our
algorithm ensures that each agent in the network only ever {\em
	explores a vanishingly small fraction $\Theta(\frac{1}{n}+\frac{\log(n)}{K})$ of the
	arms.} In other words, the sub-optimal arms do not spread and thus
not all agents need to learn and discard the sub-optimal arms.
\\

Analytically, we introduce several novel coupling arguments and tail estimates to study variants of the classical spreading processes on graphs (cf. Theorem \ref{thm:rumor_spread}, \ref{thm:discrete_rumor_2}), which can be of independent interest in themselves. Furthermore, we employ arguments based on the linearity of expectation to handle the dependencies of the regret among the agents induced by our algorithm (cf. Propositions \ref{prop:tau_cond_x},\ref{prop:tau_bounds},\ref{prop:late_good}), which we believe can be useful in studying other algorithms for our model.

\subsection{Comparison with Benchmark Systems}
\label{sec:benchmark}
\begin{figure}
	\centering
	\includegraphics[scale=0.35]{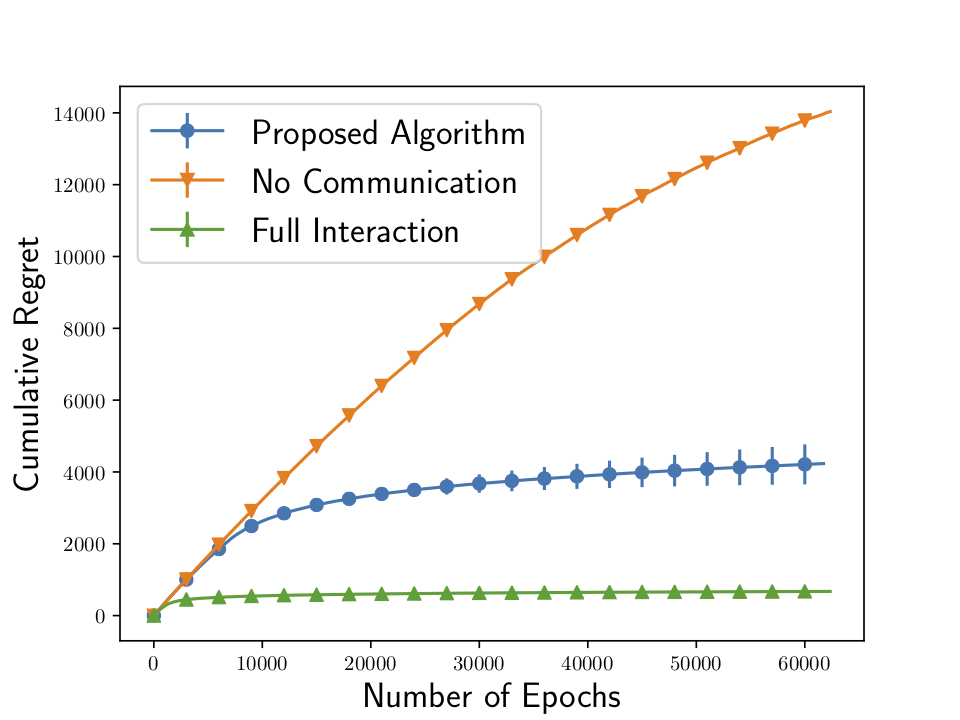}
	\caption{A plot consisting of $80$ agents and $40$ arms comparing the single agent UCB without communications with our scheme. The arm means were randomly generated in the interval $(0.4,0.85)$ and the curve is averaged over $10$ runs with $95$\% confidence interval.}
	\label{fig:regret_bound_intro}
\end{figure}

Since we are interested in quantifying the effect of collaboration through limited noisy pairwise interactions among the agents, we compare our result with the two extreme opposite scenarios
of collaboration among the agents - a setting with no communication and a one with complete interactions among agents.
\\

\noindent \textbf{1. No Communication regime} - If the players are unaware of each other and do not
interact at all, then each player will see a standard MAB
problem consisting of $K$ arms. Thus, from well known results (for ex. \cite{ucb_auer}), each
agent after playing the MAB problem for $T$ time steps, must incur a
regret that scales as 
$\BigO{  \frac{K}{\Delta} \log(T) }$. 
\\

\noindent \textbf{2. Full Interaction Regime} - On the other end is the perfect collaboration model in which every agent, whenever its clock rings, plays an arm, observes a reward and then broadcasts both the arm played and reward obtained to all other agents. In this case, every agent before playing an arm, has access to the entire system history and thus can jointly simulate a single agent optimal scheme. Thus, after a total of $T$ clock ticks of tagged agent, the total number of arm pull by all agents is roughly $nT$. It is not exact as there is some randomness in the number of times an agent plays in a given time interval determined by the randomness due the clock process of agents. Thus, the total network as a whole will incur an average regret of order $\BigO{ \left( \frac{K}{\Delta} \log(nT) \right)}$. As there are $n$ agents in total, the per agent regret in this case scales as  $\frac{1}{n}\BigO{ \left( \frac{K}{\Delta} \log(nT) \right)}$ which is of order $\BigO{ \frac{K}{n} \frac{1}{\Delta}\log(T)+ \frac{K\log(n)}{n\Delta} }$. This is the best possible per-agent regret scaling one can hope for in this networked setting and no other collaborative policy can beat this regret scaling. However, to achieve this, each agent must communicate, both its arms and the observed reward to all other $n-1$ agents, each time it plays an arm. In other words, an agent must communicate $T$ times, over $T$ plays of the arm, and this communication is broadcast to all other agents. 
\\

In our model on the other hand, we are restricted to just pairwise random communications and each agent can participate in $o(T)$ communications over $T$ times it pulls arms to collect rewards. Nevertheless, we show that our algorithm achieves, both a
significant reduction in the per-agent regret compared to the setting of no interactions among agents by bringing the order from $K$ to $\lceil \frac{K}{n} +\log(n)\rceil$ as the leading term in front of $\frac{\log(T)}{\Delta}$. Our algorithm is also only a factor $\log(n)$ off from the setting of complete interaction among agents, which has a factor of $\frac{K}{n}$ in front of the $\frac{\log(T)}{\Delta}$ term. Moreover, our algorithm achieves the reduced regret scaling with a much smaller communication resources where an agent only uses the communication channel of order $\log(T)$ times over $T$ times of play of an agent.  We plot in Figure \ref{fig:regret_bound_intro}, a representative situation showing the regret growth of our algorithm against that of the no communication and full interaction case.
\\

\textbf{Organization of the Paper} - In Section
\ref{sec:problem_setting}, we give a precise mathematical formulation
of the problem. We then specify the algorithm in Section
\ref{sec:algo} and the main theorem statement is given in Section
\ref{sec:main_result}. We then give an overview of the proof in
Section \ref{sec:proof_sketch}. We evaluate our algorithm and benchmark its performance empirically both in synthetic and real data in Section \ref{sec:simulations}. We then survey related work
in Section \ref{sec:related_work}, and then conclude with some
discussions and open problems.  The full proof of our main result is
carried out in Appendices \ref{sec:analysis}, \ref{sec:proof_early_lemma} and \ref{sec:proof_rumor}.

\section{Problem Setting}
\label{sec:problem_setting}

We
have a collection of $n$ agents, each of whom is playing the same
instance of a MAB problem consisting of $K$ arms.  The $K$ arms have unknown average rewards
$(\mu_i)_{i \in \{1,\cdots,K\}}$, where each $\mu_i \in (0,1)$. Without loss of
generality, we assume that
$1 \geq \mu_1 > \mu_2 ... \geq \mu_K \geq 0$. However, the agents are not aware of this ordering of arm-means. Denote by the arm-gap
$\Delta := \mu_1 - \mu_2$ and we shall assume that $\Delta > 0$. If at any time, any agent plays an arm $i
\in [n]$, it will receive a reward distributed as a Bernoulli random
variable of mean $\mu_i$, independent of everything else.


 \begin{figure}
	\centering
	\includegraphics[scale=0.15]{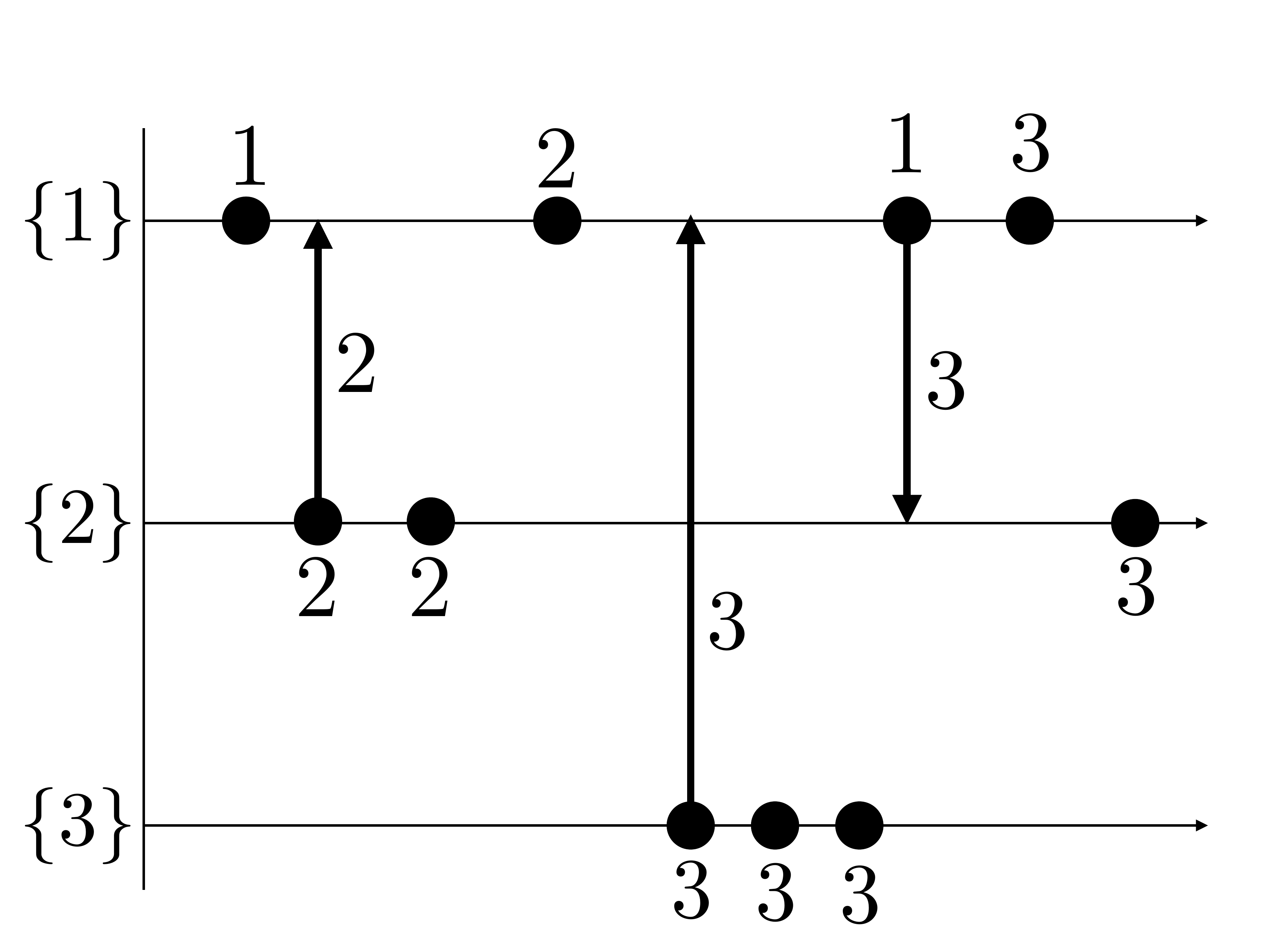}
	\caption{A schematic with $3$ arms and agents. At the
          beginning, the top, middle and bottom agents are aware of
          arms $1$,$2$ and $3$ respectively, denoted by the set. The
          dots represent epochs of the clock process at which the
          agents play an arm  and may additionally choose to
          communicate an arm. The arm-id played and communicated are
          denoted on the epochs and vertical arrows respectively. The
          recipient of the communication, denoted by the head of the
          vertical arrows are chosen uniformly at random. An agent, at
          each epoch only plays and recommends from among the arms it
          is aware of.} 
	\label{fig:model_description}
\end{figure}

\subsection{System Model}

\noindent \textbf{Clock Process} - The system evolves in continuous time, where each agent $i \in [n]$ is
equipped with an unit rate Poisson process on $\mathbb{R}_{+}$ denoted
by $C_i(\cdot)$, which functions as a \emph{clock} for agent $i$.
Each agent $i\in[n]$ takes an action only at those random time
instants when the clock $C_i(\cdot)$ `ticks', i.e., only at those
random times $t \geq 0$ such that $C_i(t) - C_i(t^{-}) = 1$. The times
$t$ when a clock ticks is referred to as an \emph{epoch} of the clock
process.  The processes $(C_i(\cdot))_{i \in [1,n]}$ are all i.i.d.,
and hence the actions of different agents are not
synchronized. 
\\



%

\noindent \textbf{Agent's Actions} - An action by an agent (which it makes at the epochs of its clock
process) consists of three quantities - $(i)$ an arm among the set of
$K$ arms to play and obtain a reward (where the observed reward is
either a $0$ or a $1$), $(ii)$ the choice of whether to initiate a
pairwise communication, $(iii)$ and if so what message and to whom  to
communicate to. The message communicated by any agent, each time it does, the message length must not exceed $\BigO{\log(nK)}$ bits (in our algorithm, message lengths are smaller than $\log_2K +1$ bits). Furthermore, the message length must not either scale with time or depend on the problem parameters such as arm means or gap $\Delta$. Moreover, over $T$ total epochs of an agent where it played arms and collected rewards, it must have communicated only $o(T)$ times. From henceforth, we use the term \emph{number of epochs} to denote the number of times an agent has played arms and collected rewards and \emph{time} to refer to the continuous time during which the agents' clocks ring. Our system is \emph{decentralized}, namely agents' actions of which  arm to pull and whether to communicate and if so what and whom to communicate to must only depend on the agent's past history or arms pulled, rewards obtained and messages received.
\\

\noindent \textbf{Technical Setup} - We suppose there exists a probability space
$(\Omega,\mathcal{F},\mathbb{P})$, which contains $n$ i.i.d, unit rate
\emph{marked} Poisson Point Processes (PPP), corresponding to the clocks for the agents.
Each epoch of each
clock, has associated with it, three independent uniform
$[0,1]$ valued random variables. The system's sample path is then a
measurable (i.e., deterministic) function of the set of marked
PPPs. The interpretation of this setup is as follows. Every agent $i \in \{n\}$, plays
an arm at the epochs of its clock process and the marks decide actions
(whether to communicate and which arm to play) and their outcomes
(observed rewards and recipients of communication if any). The action of every
agent at every epoch of its clock must be measurable function of
\emph{only} its arms played, observed rewards and received messages in
the past. In the absence of messaging, every agent is playing a
standard MAB problem, where its action, which is just which arm to
play, is a measurable function of the past arms chosen and rewards
obtained. The key new ingredient in our setup is the \emph{active
  messaging}, where agents can choose, based on the history of chosen
arms, observed rewards \emph{and} received messages, the arm to play
and the message to communicate, if at all. Thus, our
setting is distributed since an agent is not aware of the arms played
and the rewards obtained by other agents, but only has an indirect
knowledge through the active messages received.
\\







%

\subsection{Performance Metric}
The main performance metric
of interest is the cumulative regret incurred by all agents. For any
agent $i \in \{1,\cdots,n\}$, and $m \in \{1,\cdots,\}$, denote by
$\mathbb{I}_i^{(m)} \in \{1,\cdots,n\}$ to be the arm played by
agent $i$, in its $m$th epoch. For any agent $i$, after it
has played for $T$ epochs, define by
\begin{align*}
R_i^{(T)} =  \sum_{t=1}^{T} (\mu_1 - \mu_{{I}_i^{(t)}}).
\end{align*}
 In this
multi-agent scenario, we want to design algorithms, in which every agent $i \in \{1,\cdots,n\}$, is interested in minimizing its own cumulative regret $ \mathbb{E}[R_i^{(T)}]$, where the expectation is with respect to both the observed randomness
and the policy, while requiring as minimal a communication
resources as possible. 


\subsection{Model Assumptions}

Each agent can agree upon a common protocol to follow prior to execution. This could potentially depend on the agent's indices. We assume all agents are aware of a common non-trivial lower
bound $\varepsilon$ on this arm-gap $0 < \varepsilon \leq \Delta$, and
use this information to make decisions.  Nevertheless, our proposed algorithm still executes if $\varepsilon > \Delta$, and we verify that the degradation in performance of our algorithm is minimal in this case through simulations in Section \ref{sec:simulations}. 
\\

Such an assumption of known
$\Delta$, but unknown mean rewards (which is the setting in our case),
is used in several MAB settings (see the book of \cite{lattimore_book}) - for instance the classical $\epsilon$-greedy
algorithm \cite{sutton} or the UCB-A algorithm
\cite{best_arm_bubeck}. In the networked setting similar to ours, this assumption seems to be standard (\cite{p2p_simple_regret},\cite{hillel}). Certain algorithms in \cite{hillel} require an input parameter $T$, that depends on the arm-gap $\Delta$. However, it is known from \cite{lai_robbins},\cite{bubeck_surveyt}, that even if the forecaster knows the arm-gap $\Delta$, the regret scales at-least as order $\Theta\left( \frac{\log(T)}{\Delta}\right)$ \cite{lai_robbins}. Thus, the knowledge of arm-gap, does not affect the complexity of the problem, at-least from the perspective of regret scaling in time.

%% file: proof_sketch.tex
\section{Proof Sketch}
\label{sec:proof_sketch}

 We identify certain nice behaviour which occurs with high probability (w.h.p). Set {\color{black}$\delta = \frac{1}{3M}$}. We call the system  {\bf Good}, if the following events occur.  

\begin{itemize}
	\item \emph{Event }$\mathcal{E}_1$ - All agents are aware of the best arm by time $(M-1)L(1 + \delta)$.
	\item \emph{Event }$\mathcal{E}_2$ - The total number of times any agent is ever contacted by another agent $j$ when $j$ is in the early phase, i.e., $j$ is in state $-1$ or lower is at-most $2M-2$.
	\item \emph{Event }$\mathcal{E}_3$ - By time $(M-1)L(1+\delta)$, all agents are in phase $-1$ or lower.
\end{itemize}



Notice that every agent will play the best arm in phases $0$ and beyond, if the Good event holds, as all agents are aware of the best arm by at-most phase $-1$. We will show in Lemma \ref{lem:early_good},
that the system is Good w.h.p.




\subsection{Late-Stage Analysis}

We split the regret as the sum of three terms - $(i)$ Regret in the early phase, which will be linear as agents are only doing best-arm identification, $(ii)$ - Regret in the late-phase due to playing the UCB algorithm with the doubling trick and $(iii)$ Linear regret in the late-phase until an agent becomes aware of the best arm, if it is not aware of the best arm at the beginning of the late phase. The first term is trivial as we will assume that all agents incur a worst case regret of $1$ in each of its early phase epoch. The main challenge in computing the second term is that the number of arms an agent is aware of in any late-stage is a random variable and not fixed. However, the regret of an agent conditional on the number of arms is easy to compute, as it follows directly from \cite{ucb_auer}. The key idea here is to notice that conditioning on the number of arms an agent is aware of at the beginning of a phase, has no effect on the regret incurred by an agent during the phase in consideration. This is so as we do not re-use samples across phases to keep track of estimates on arm means. As the regret conditional on the number of arms, scales linearly in the number of arms, it suffices to separately evaluate just the mean number of arms an agent is aware of at the beginning of a phase. This is done in Propositions \ref{prop:late_good} and \ref{prop:arm_count}. To evaluate the third term, we upper bound the time it takes for an agent to learn the best arm by the time it takes by agent $1$ to recommend the best arm. We show this in Propositions \ref{prop:dist_X}, \ref{prop:tau_bounds}, \ref{prop:S_Value} and \ref{prop:S_Value_integer}, that the average number of epochs an agent has to wait in the late-stage before being recommended the best arm by agent $1$ is `small'.

\subsection{Early-Stage Analysis}
We establish in Lemma \ref{lem:early_good} proven in Section \ref{sec:proof_early_lemma}, that the system is Good with probability at-least $1-150n^{-3}$. The probabilities of events $\mathcal{E}_2$ and $\mathcal{E}_3$  are straightforward to deduce from Chernoff tail bounds which we do in Lemmas \ref{lem:early_e2} and \ref{lem:early_e3} respectively. Concluding about the probability of event $\mathcal{E}_1$ is the key technical innovation where the difficulty stems from the following reason. We need to first condition on event $\mathcal{E}_2$, as that will imply that all agents in the early stage make a recommendation from among at-most $2M$ other arms.
By the choice of $L$, and known results from \cite{bubeck_ucb_simple_regret} reproduced as Lemma \ref{lem:bandit_result_perr},  conditional on event $\mathcal{E}_2$, agents that possess the best arm recommend it with probability at-least $99/100$. However, conditioning on event $\mathcal{E}_2$ induces correlations on the agent ids that receive the messages and hence makes the spreading process difficult to analyze directly, as the recipients are no longer independent conditional on $\mathcal{E}_2$. 
\\
 
We proceed by considering and analyzing a fictitious virtual system which is identical to our algorithm in the early stage with a crucial modification that agents in this fictitious system will \emph{drop} arms if it at any point it is aware of $2M+1$ or more arms. However, agents in this virtual system will not drop the best arm once it becomes aware of it. Note that this is only a mathematical stochastic process under consideration and hence we can assume that the agents in this virtual system know the best arm's index. We show in Lemma \ref{lem:coupling}, that w.h.p., this virtual system has identical sample paths as our algorithm upto time $(M-1)L(1+\delta)$.
\\

We study the virtual system by a reduction to a discrete time  rumor mongering process. Specifically, we will establish in Lemma \ref{lem:discrete}, that agents in this virtual system are `in sync', i.e., for all $j \in [M]$, no agent makes its $(j+1)$th recommendation, before all other agents finish making their $j$th recommendation (See also Figure \ref{fig:early_stage_discrete}). We notice that the discrete rumor mongering process we obtain is a variation of the classical rumor spreading on \cite{discrete_rumor} and \cite{pittel}, with two important distinctions. First, in our discrete time model, an agent only spreads the rumor after a one time slot delay after receiving the rumor. Second, each spreading attempt of an agent in each time slot, is successful with probability $99/100$, as opposed to always being successful in \cite{discrete_rumor}. We show in Theorem \ref{thm:rumor_spread}, that the total spreading time for this process is order $\log(n)$ with high probability. We provide a simple proof of the spreading time in Theorems \ref{thm:rumor_spread} and \ref{thm:discrete_rumor_2}, which could be of independent interest. This enables us to conclude that event $\mathcal{E}_1$ holds w.h.p. for the virtual system, which in turn implies it holds w.h.p. for our algorithm, as the virtual system and our algorithm have identical sample paths upto time $(M-1)L(1+\delta)$ w.h.p..



%

 \begin{figure}
	\centering
	\includegraphics[scale=0.2]{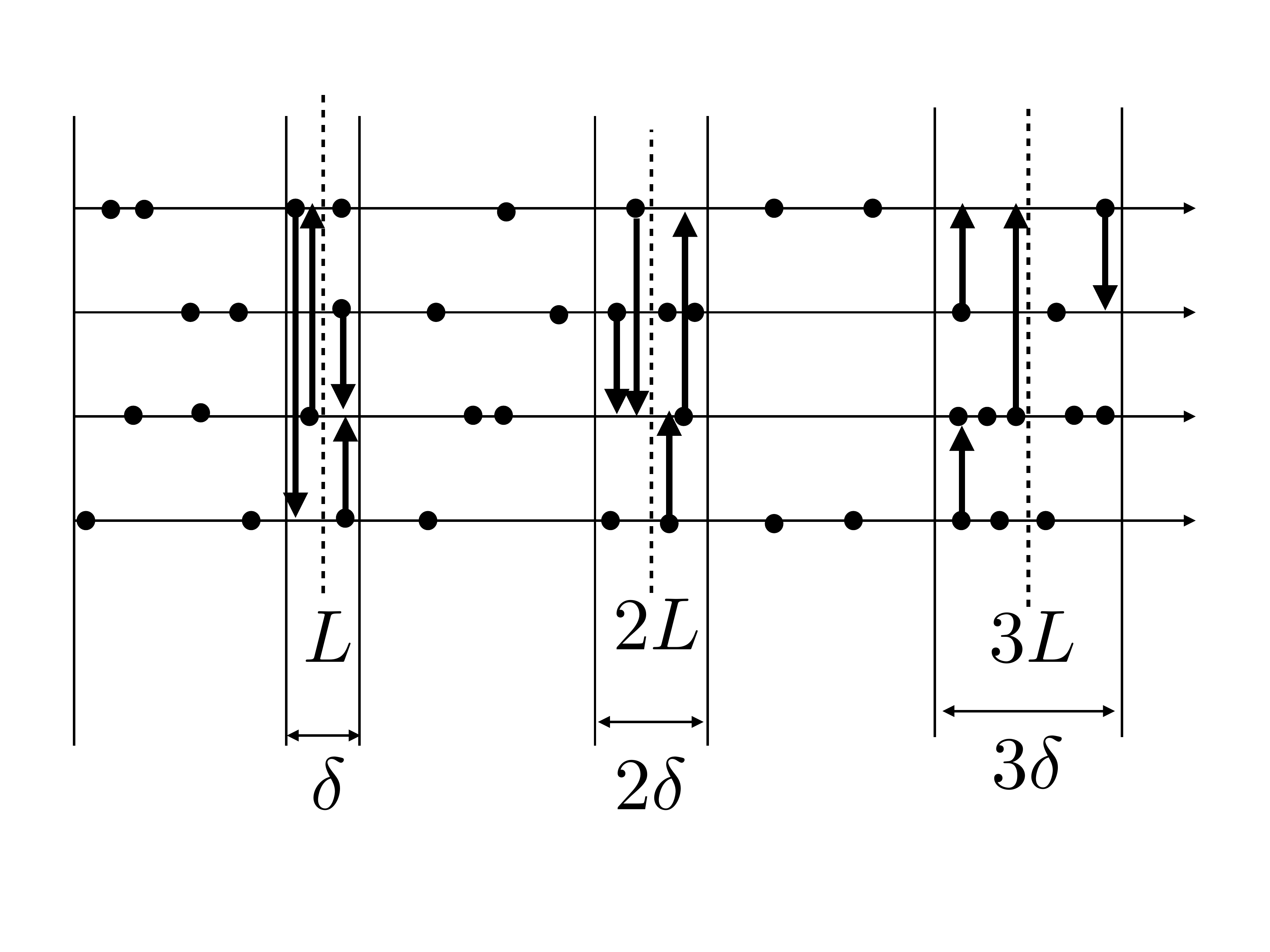}
	\caption{A schematic depicting the almost-discrete behaviour of the algorithm, where agents make a recommendation every $3$rd epoch. We will establish in Lemma \ref{lem:discrete}, that w.h.p., all agents make their $i$th early stage recommendation in the interval $[il(1-\delta),iL(1+\delta)]$, for all $i \in [1,M]$.} 
	\label{fig:early_stage_discrete}
\end{figure}

%% file: simulations.tex
\section{Numerical Results}
\label{sec:simulations}

We empirically evaluate the performance of our algorithm and in particular highlight the gains due to collaboration in reducing per-agent regret. Throughout this section, we use $M = \lceil 2.5\log(n) \rceil+1$ and $\lceil L = 0.8\frac{2M + \lceil \frac{K}{n} \rceil}{\varepsilon^2} \log(20(2M + \lceil \frac{K}{n} \rceil)) \rceil$. This is different from that mentioned in our Theorem \ref{thm:known_delta} as the constants there arise from certain tail probability bounds that are not tight. 

\subsection{Synthetic Data}

We evaluate the performance of our algorithm in Figure \ref{fig:synthetic}. For each case of $\Delta,n$ and $K$, we  sample the arm means uniformly in the range $(0.4,0.85-\Delta)$ and the best arm has mean $0.85$. To be comprehensive, we test our algorithm with instance settings $\Delta \in \{0.1,0.2\}$ and the number of arms and agent pairs of $(n,k) = \{(20,50),(30,60), (40,60)\}$. We vary the input parameter $\varepsilon$ of our algorithm and compare the performance of our algorithm against the two benchmarks stated in Section \ref{sec:benchmark}, namely a system with no interaction and a system with perfect interaction. The no interaction system corresponds to a single agent playing the MAB following the UCB($2$) algorithm of \cite{ucb_auer}. The perfect interaction benchmark is one wherein when an agent's clock ticks, it has access to the entire system history and chooses an arm according to the UCB($2$) algorithm using the entire history. In each plot, we first sample the arm means and then do $10$ random runs and plot the average over these runs along with $95$\% confidence intervals.  
\\

\noindent \textbf{Results} - We see from Figure \ref{fig:synthetic} that our proposed algorithm, is both practically scalable to large systems and effective in leveraging the collaborations to significantly reduce the per-agent regret compared to the case of no collaborations. Even with small $\varepsilon$, our algorithm has much smaller regret growth eventually compared to the setting of no collaboration.
Moreover, there is still substantial performance gain in regret when the input parameter $\varepsilon$ of our algorithm is varied. Note that the theoretical guarantees in Theorem \ref{thm:known_delta} only holds if $\varepsilon < \Delta$ while in practice (as seen in Figure \ref{fig:synthetic}) our algorithm performs well even if $\varepsilon > \Delta$.

\begin{figure*}

	\includegraphics[scale=0.3]{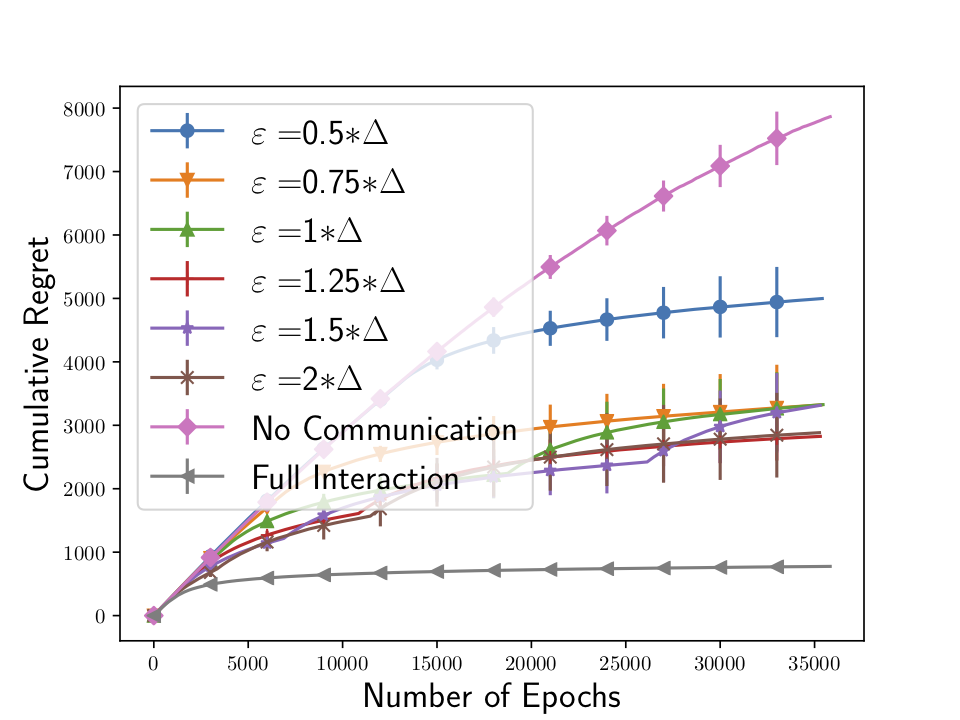}
\includegraphics[scale=0.3]{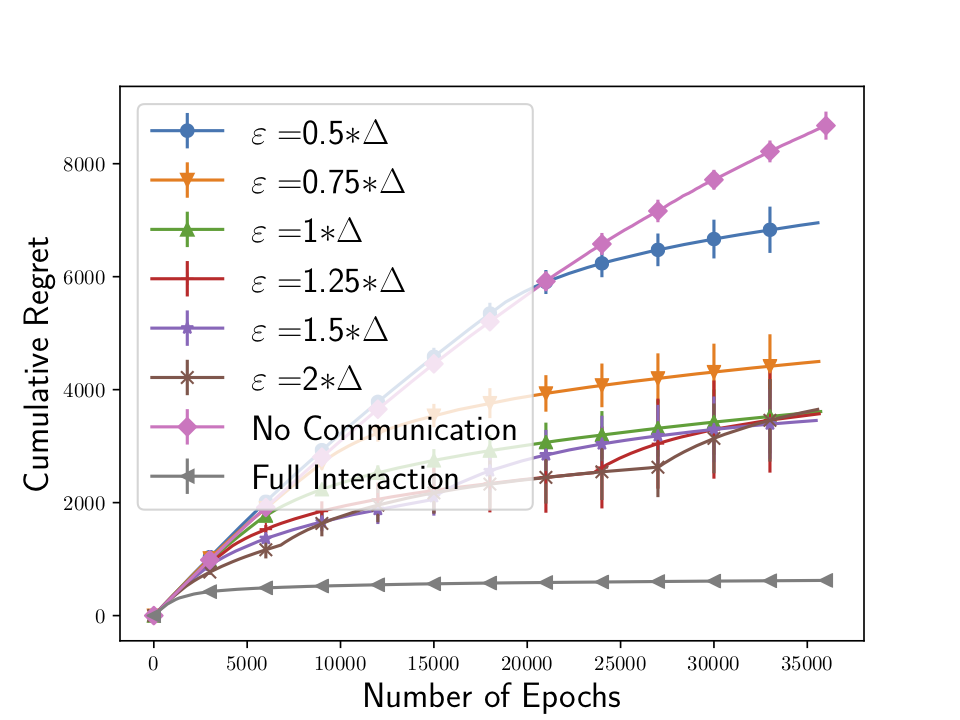}
\includegraphics[scale=0.3]{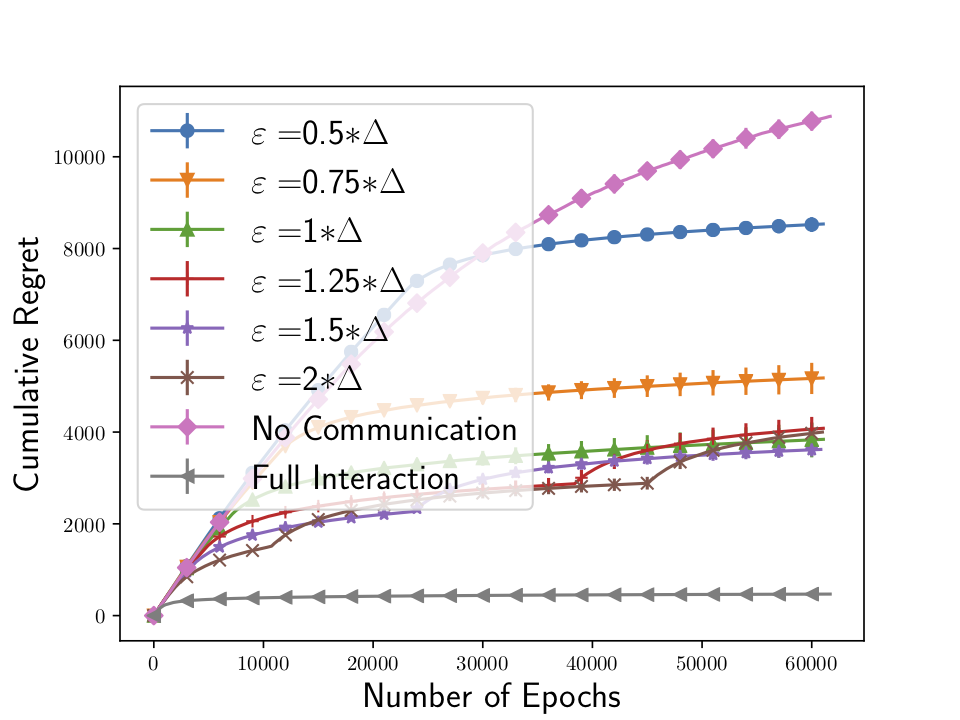}

\includegraphics[scale=0.3]{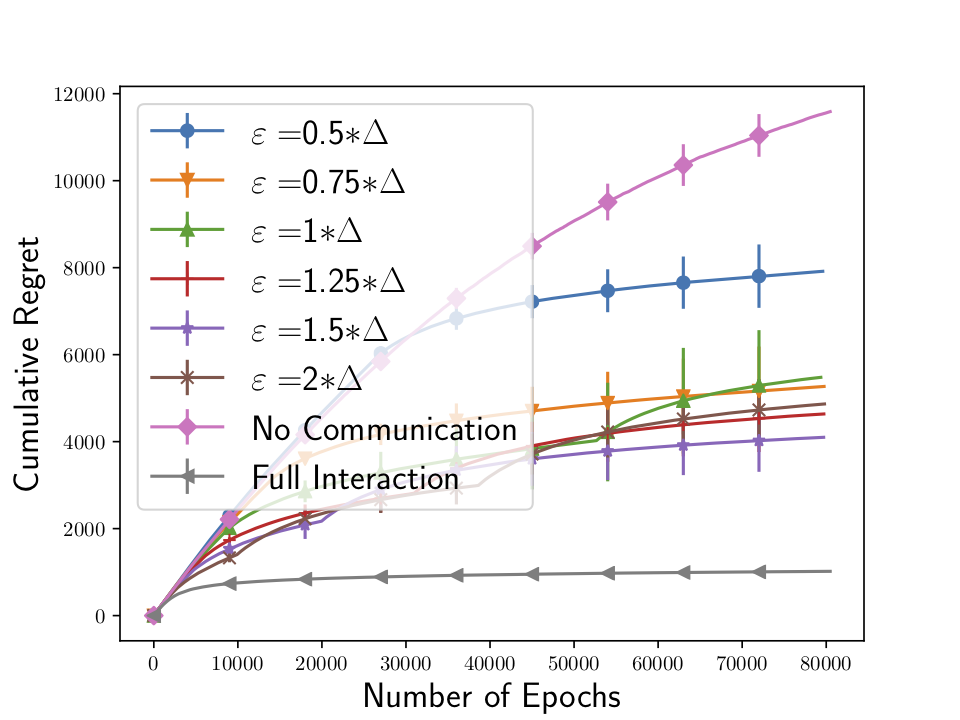}
\includegraphics[scale=0.3]{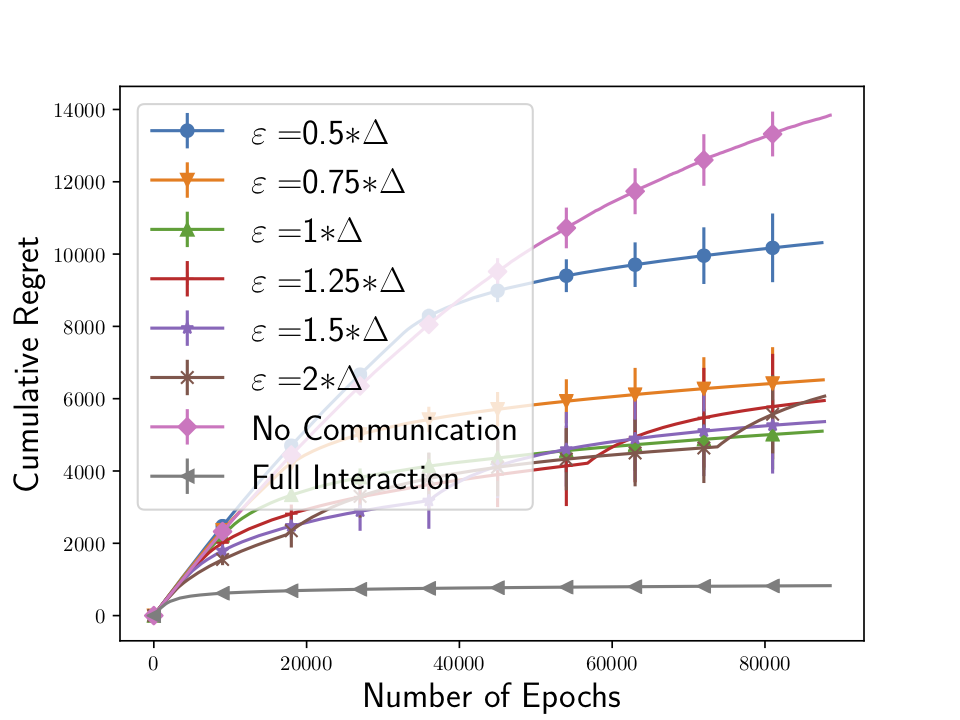}
\includegraphics[scale=0.3]{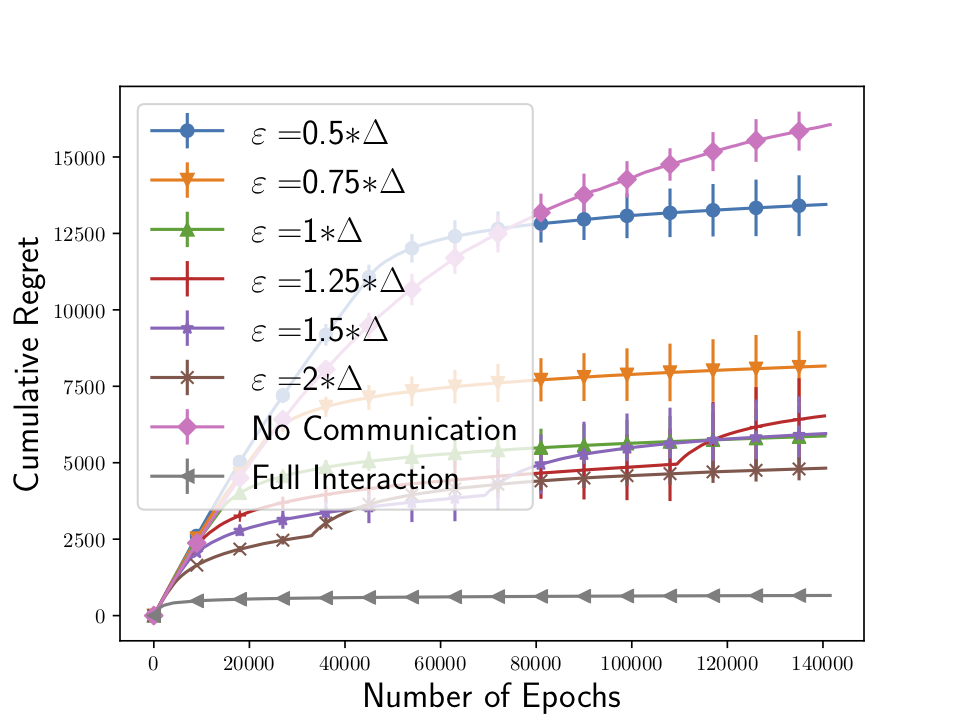}
	
	\caption{ The system parameters $(n,K)$ from left to right are $(20,50),(30,60),(40,60)$ respectively. The top row corresponds to $\Delta = 0.2$ and the bottom row $\Delta = 0.1$.}
	\label{fig:synthetic}
\end{figure*}

\begin{figure*}

	\includegraphics[scale=0.3]{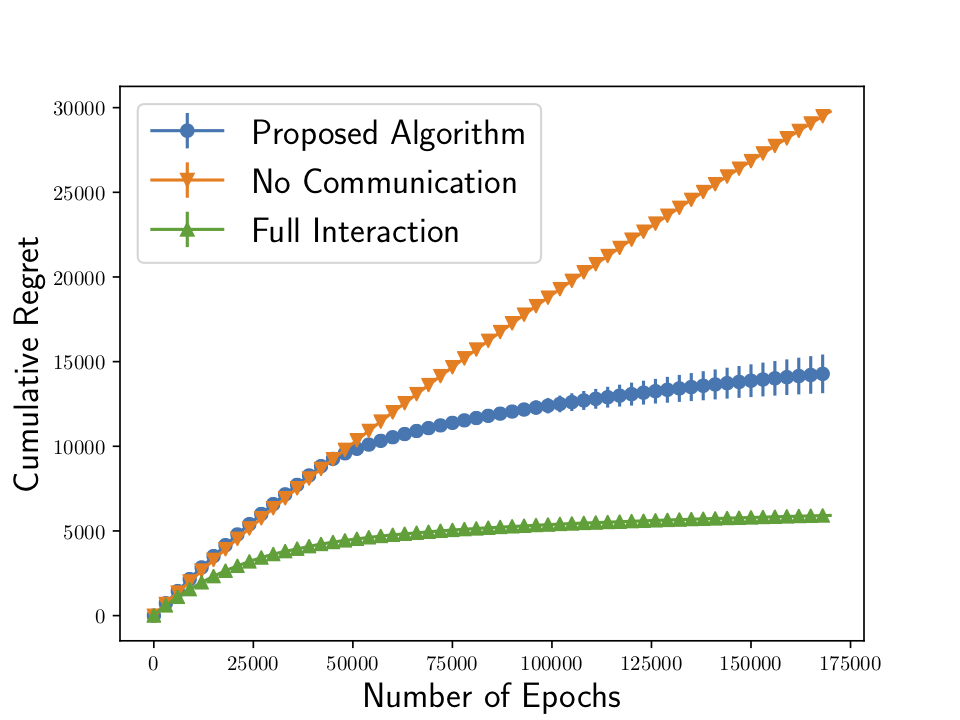}
\includegraphics[scale=0.3]{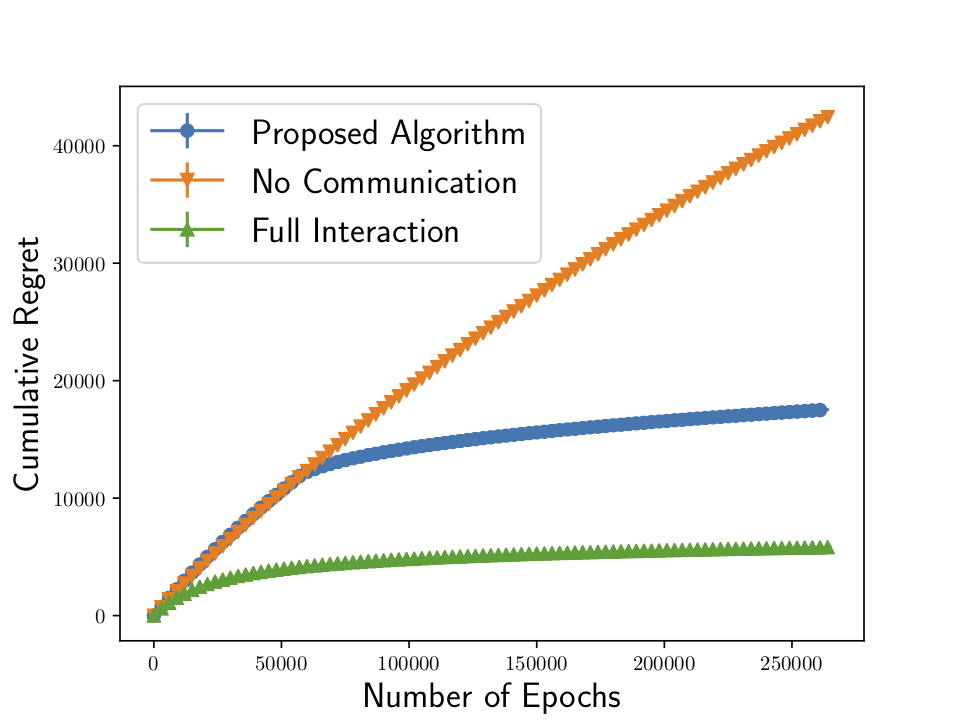}
\includegraphics[scale=0.3]{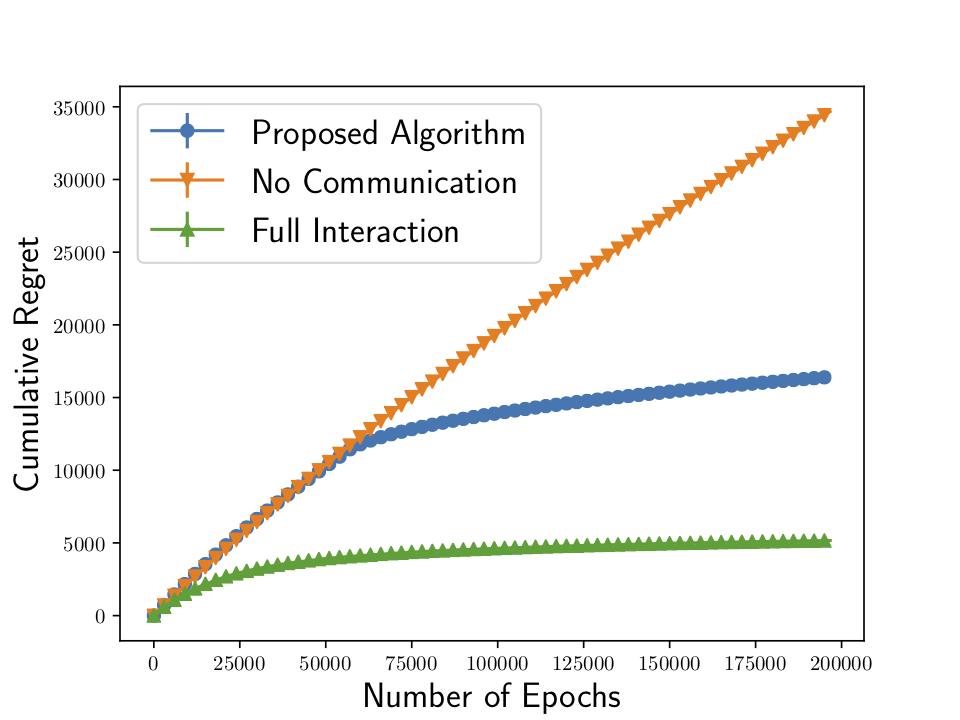}
	
	\caption{The plot of regret in Movielens data. The figures from left to right comprise systems with $(n,K)$ as $(20,200),(25,220),(30,240)$ respectively. The average number of arms an agent was aware of by the end of simulation window were $15.5,14.8,14.1$ respectively}
	\label{fig:real_data}
\end{figure*}

\subsection{Simulations with Real Data}

We consider the Movielens $1M$ data \cite{movielens} to run our algorithm. This dataset has $4k$ users and $6k$ movies. We selected a user category, corresponding to same gender, age and occupation. We ensured that there are at-least $150$ users in each category. We then considered a subset of movies such that each user rated at-least $30$ of those movies and each movie is rated by at-least $30$ of these users. We extract out this submatrix and run standard matrix completion \cite{fancyimpute} to fill the missing details. We then averaged each column and divided this average by $5$. This then forms the mean rating normalized to $[0,1]$ of this movie in this user group. This set of normalized scores for movies are used as arm-means, where each movie corresponds to an arm. 
In figure \ref{fig:real_data}, we run our algorithm with this arm-means and a common parameter of $\varepsilon = 0.05$. In each plot of Figure \ref{fig:real_data}, we randomly sample a collection of movies satisfying the above property, and then do $10$ random runs and plot the average over these runs along with $95$\% confidence intervals. The confidence bars are smaller than the size of markers on plot.
\\

\noindent {\bf Results} - We see from Figure \ref{fig:real_data}, that even for large systems, our algorithm reaps benefits of collaboration. In particular, since the number of arms is large ($200$ or more), single agent UCB is incurring linear regret in the simulation window, while our algorithm has gone into the late phase and has a sub-linear regret growth much earlier. This is because, in all experiments our algorithm is only exploring much smaller number of sub-optimal arms (under $16$ in all cases as described in Figure \ref{fig:real_data}) compared to the standard UCB. Moreover, the arm gap in all of the plots are $0.01$ or smaller (note the arms were randomly selected for each plot), yet our parameter of $\varepsilon = 0.05$ performs quite well, implying that our algorithm is quite robust.

%% file: analysis_abishek_scratch_changed.tex
\section{Analysis of the Algorithm }
\label{sec:analysis}

We will analyze the regret experienced by fixing an arbitrary agent $I \in \{2,\cdots,n\}$. Recall the definition given in Section \ref{sec:proof_sketch} of when we call the early-phase of our system {\bf Good}. Observe that if the system is Good, then every agent will be aware of the best arm, in phase $-1$. Thus, conditional on the event Good, all agents will start playing the best arm in phases $0$ and above. For ease of notation, denote by $\mathcal{T}_n := (M-1)L(1 + \delta)$ in the rest of the proof.

\begin{lemma}
	A sample path is Good with probability at-least {\color{black}$1-150\log(n)n^{-3}$}.
	\label{lem:early_good}
\end{lemma}

The proof of this lemma is deferred to Section \ref{sec:proof_early_lemma}. To carry out the analysis further, we will need two classical results from the study of Multi-Armed Bandits (MAB) \cite{ucb_auer, bubeck_ucb_simple_regret}.

\begin{proposition} \cite{ucb_auer}
Consider playing the UCB($\alpha$) algorithm for $T$ time steps of a $K$ armed MAB. The regret is upper bounded by $\mathbb{E}[R_T] \leq \frac{4 \alpha}{\Delta} \log(T) + K \left( 1 + \frac{\pi^2}{3} \right)$.
\label{prop:ucb_regret}
\end{proposition}

We will also need another result from the literature \cite{bubeck_ucb_simple_regret} that we reproduce here for completeness. 
\begin{proposition} \cite{bubeck_ucb_simple_regret}
	Consider a MAB problem with $K$ arms and playing the UCB strategy. The probability that after $T$ time steps the best arm is not the most played arm is at-most $\frac{K}{\alpha-1}\left(\frac{T}{K}-1\right)^{2(1-\alpha)}$, for all $T$ such that  $T \geq \max \left( K(K+2), K \left( 1 + \frac{4 \alpha \log(T)}{\Delta^2}\right) \right)$.
	\label{prop:bandit_late}
\end{proposition}
\begin{remark}
	The constant $T_0$ is chosen such that  $T_0 \geq \max \left( K(K+2),K \left( 1 + \frac{4 \alpha \log(T_0)}{\Delta^2}\right) \right)$, and hence the previous error bounds are applicable to all agents in phases $0$ and beyond.
\end{remark}

To now carry out the analysis, we define a few other random variables. Denote by $\tau$ to be the 
number of epochs of agent $1$ that have elapsed, before agent $I$ becomes aware of the best arm (i.e., arm indexed $1$). Recall that agent $1$, has the best arm in its set at time $0$, i.e., $1 \in A_1^{(-M)}$. Denote by the random variable $X \in \{1,2,..\}$, to be the first phase of agent $1$, when agent $1$ communicates the best arm to agent $I$ in consideration. In other words, $X$ is a random variable denoting the earliest late-phase state of agent $1$, such that $O_1^{(X)} = 1$, i.e., agent $1$ has for its opinion the best arm, and agent $1$ communicates this opinion to agent $I$, while it is in phase $X$. Denote by $S \in \{-M,\cdots,,0,1,...\}$, to be the state in which agent $I$ receives the best arm for the first time, as a recommendation from another agent.


\begin{proposition}
	\begin{align*}
\mathbb{E}[\tau | X]  \leq ML +  \lfloor \frac{T_0}{2}2^{2^{X-1}} \rfloor + n
	\end{align*}

\label{prop:tau_cond_x}
\end{proposition}
\begin{proof}
	The time $\tau$ is clearly upper bounded by the time agent $1$ takes to spread the best arm itself to agent $I$. From the definition of the random variable $X$, this happens at some point of time when agent $1$ is in state $X$. Conditional on $X$, the number of epochs of agent $1$ taken to reach the end of phase $X-1$ (which is also equal to the beginning of phase $X$) is $ML +  \lfloor \frac{T_0}{2}2^{2^{X-1}} \rfloor $. Now, in phase $X$, the average number of epochs taken by agent $1$ to communicate its opinion to agent $I$ is at-most $n$. This is at-most $n$, since conditional on $X$, we know that agent $1$ will communicate the best arm within a deterministic number of epochs. Since, the average time of a Geometric random variable conditioned that it is smaller than a fixed deterministic constant is at-most its mean, in an additional average of $n$ epochs of agent $1$ in phase $X$, it will communicate the best arm to agent $I$. 
%
%
%
%
\end{proof}

\begin{proposition}
	For all $j \geq 1$, we have 
	\begin{equation*}
	\mathbb{P}[X > j] \leq \prod_{i=1}^{j}  \left( \frac{K}{\alpha-1}\left(\frac{T_{i-1}}{K}-1\right)^{2(1-\alpha)} + e^{-2^{i}} \right) , 
	\end{equation*}
	where $T_i : = \lfloor \frac{T_0}{2} 2^{2^i} \rfloor - \lfloor \frac{T_0}{2} 2^{2^{i-1}} \rfloor$. Here the empty product $\prod_{i=1}^{0} = 1$. 
	\label{prop:dist_X}
\end{proposition}
\begin{proof}

 To have the event $X=j$, in all phases $l \in \{1,\cdots,j-1\}$, we must have either had the opinion $O_1^{(l)} \neq 1$, or agent $1$ does not communicate the best arm to agent $I$ in phase $i$. Additionally in phase $j$, both the opinion $O_1^{(j)}$ must correspond to the best arm and agent $1$ must have communicated it to agent $I$ in its $j$th phase. Since we are interested in an upper bound on the probability, we can assume that agent $1$ is aware of all $n$ arms in all its late-stages. This provides the largest error probability that the opinion of agent $1$ in a late-phase is different from the best-arm. From Proposition \ref{prop:bandit_late}, we know the probability that agent $1$ has an opinion in phase $i$ which is different from the best-arm is at-most $\frac{K}{\alpha-1}\left(\frac{T_{i-1}}{K}-1\right)^{2(1-\alpha)} $. Similarly, the probability that agent $1$ fails to communicate the best arm to agent $I$ in $n 2^i$ attempts is at-most $(1-n^{-1})^{n2^i} \leq e^{-2^i}$. Thus the probability that agent $1$ fails to inform agent $I$ of the best arm, when agent $1$ is in phase $i$ is at-most 
 $ \left( \frac{K}{\alpha-1}\left(\frac{T_{i-1}}{K}-1\right)^{2(1-\alpha)} + e^{-2^{i}} \right)$. The result then follows from the independence of opinions and the communication recipients of agent $1$ across different phases and epochs. 
\end{proof}

Notice immediately that we have $\mathbb{P}[X < \infty] = 1$, and thus the algorithm ensures that agent $I$ (and by symmetry) all agents will be aware of the best arm eventually with probability $1$. However, we want to ensure that agents become aware of the best arm `soon' enough on average, which is the subject of the following computations.

\begin{proposition}
	\begin{align*}
\mathbb{E}[\tau | \textrm{Good}] &\leq ML, \\ \mathbb{E}[\tau | \lnot \text{Good}] &\leq ML + 5T_0 + n
	\end{align*}
\label{prop:tau_bounds}
\end{proposition}
\begin{proof}
	Conditional on the system being $\textrm{Good}$, we know that all agents are aware of the best-arm before any agent moves into the late-phase. Since, every agent moves into the late-phase after $ML$ epochs, the first inequality follows. 
	\\


	For the second Equation, we proceed as follows. We upper bound the number of epochs $\tau$ by the number of epochs that agent $1$ takes to spread the best arm to agent $I$ in the late-phase of agent $1$. Conditional on the event not Good, we assume a worst case upper bound, where agent $1$ is playing among all the $K$ arms in all its late-phases. Agent $1$ moves into the late-phase after $ML$ clock epochs. We thus only need to compute the average number of epochs agent $1$ takes, before it finishes phase $X-1$ in the late-stage. For any late-phase $j \geq 1$, we know a bound on  $\mathbb{P}[X = j+1]$ from Proposition \ref{prop:dist_X}. In the event $X=j+1$, agent $1$ takes a total of $ \lfloor \frac{T_0}{2} 2^{2^j} \rfloor$ epochs to move from the beginning of phase $0$ to the beginning of phase $j+1$. Moreover, once in phase $X$, agent $1$ will spread its opinion to agent $I$ in at-most $n$ average epochs. This follows since each recipient of recommendations are chosen uniformly at random independent of everything else, and thus average number of epochs required to contact agent $I$ is $n$. Moreover, we know that within $n2^X$ epochs, agent $1$ will communicate with agent $I$. This conditioning only reduces the average number of epochs required from $n$. Thus, the expected number of epochs of agent $1$ to get from the beginning of phase $0$ to the beginning of phase $X$ is at-most {\color{black}
	\begin{align*}
	\mathbb{E}[\tau | \lnot \textrm{Good}] &\leq ML + \sum_{j=1}^{\infty}  \lfloor \frac{T_0}{2} 2^{2^{j}} \rfloor \mathbb{P}[X=j] + n, \\
	&\leq  ML +  \sum_{j=1}^{\infty}  \lfloor \frac{T_0}{2} 2^{2^j} \rfloor \mathbb{P}[X\geq j] + n ,\\
	& \leq ML + 2T_0 + \sum_{j=3}^{\infty}  \lfloor \frac{T_0}{2} 2^{2^j} \rfloor \mathbb{P}[X\geq j] + n.
	\end{align*}
	From Proposition \ref{prop:dist_X}, we can bound the last series sum term as 
	\begin{align*}
	\sum_{j=3}^{\infty} \lfloor \frac{T_0}{2} 2^{2^j} \rfloor \prod_{i=1}^{j} \left( \frac{K}{\alpha-1}\left(\frac{T_{i-1}}{K}-1\right)^{2(1-\alpha)} + e^{-2^{i}} \right) &
	\leq \sum_{j=3}^{\infty} \lfloor \frac{T_0}{2} 2^{2^j} \rfloor \left( \frac{K}{\alpha-1}\left(\frac{T_{j-1}}{K} \right)^{2(1-\alpha)} + e^{-2^{j}} \right)\\
	&\leq \frac{T_0}{2}\sum_{j=3}^{\infty} \left( \frac{2}{e}\right)^{2^j} + \frac{T_0}{K} \sum_{j=3}^{\infty}\frac{1}{2^{2^j}}, \\
	&\leq T_0
	\end{align*}

} To have this sum convergent is precisely why agents communicate for $n2^j$ times in phase $j$ in our algorithm. This allows the error probability of $e^{-2^{j}}$, to decay doubly exponential, to make the above sum convergent. 
%
%
%
%
\end{proof}
As a corollary of the above statement, we get the following.

\begin{corollary}
	Denote by ${\tau}_I$ to be the number of epochs of agent $I$, before it is aware of the best arm. Then,
	\begin{align*}
	\mathbb{E}[{\tau}_I|\textrm{Good}] &\leq ML \\
	\mathbb{E}[{\tau}_I| \lnot \textrm{Good}] &\leq ML + 3T_0 + n
	\end{align*}
	\label{prop:S_Value}
\end{corollary}
\begin{proof}
Observe that the clock processes across agents are i.i.d.. The random variable $\tau$ is independent of the clock process $C_I(\cdot)$. More importantly, the random variable $\tau$ is independent of the inter-epoch duration process of $C_1(\cdot)$, and only depends on the randomness of the independent marks of $C_1(\cdot)$. Since, for any random variable $F \in \mathbb{N}$ such that $F$ is independent of $C_I(\cdot)$ and $C_1(\cdot)$, the expected number of epochs in $C_I(\cdot)$, when $F$ epochs occurs in $C_1(\cdot)$ is $F$, the proof follows from Proposition \ref{prop:tau_bounds}.
\end{proof}

\begin{proposition}
	Denote by $S \in \{-M,\cdots,,0,1,...\}$ be the random variable denoting the phase of agent $I$, when  agent $I$ receives the best arm. Then $\mathbb{E}[S \vert \lnot \textrm{Good}] \leq 2$ and $\mathbb{E}[S| \textrm{Good}] \leq -1$.
\label{prop:S_Value_integer}
\end{proposition}
\begin{proof}

	From definition of $\tau_I$, we know from Corollary \ref{prop:S_Value}, that $\mathbb{E}[\tau_I] \leq ML +  5T_0 + n$. For any deterministic $t \in \mathbb{N}$, denote by $\mathcal{S}_t^{(I)} \in \{-M,\cdots,0,1,\cdots\}$ to be the state of ant agent after $t$ epochs. From the description of the algorithm, we have

	\begin{align*}
	\mathcal{S}_t \leq \inf \left\{ m \geq 0 : t \leq   \lfloor \frac{T_0}{2}2^{2^m} \rfloor \right\},
	\end{align*}
	 It is easy to verify that for $t = ML + 5T_0 + n \leq 8T_0$, that $\mathcal{S}_t \leq 2$. Thus, after a random $\tau_I$ number of epochs, we have 
	\begin{align*}
	\mathbb{E}[S]  &= \mathbb{E}[\mathcal{S}_{\tau_I}]  \\
	&= \mathbb{E}\left[  \inf \left\{ m \geq 0 : \tau_I \leq   \lfloor \frac{T_0}{2}2^{2^m} \rfloor \right\} \right]\\
	&\stackrel{(a)}{\leq}  \inf \left\{ m \geq 0 : \mathbb{E}[\tau_I] \leq  \sum_{l=0}^{m  }  \lfloor \frac{T_0}{2}2^{2^l} \rfloor \right\} \\
	&= \mathcal{S}_{ ML + 5T_0 + n} \leq 2. 
%
	\end{align*}
	Inequality $(a)$ follows from the fact that $\mathbb{E}[\inf f(X)] \leq \inf \mathbb{E}[f(X)]$ for any non-negative function $f(\cdot)$.
\end{proof}

\begin{proposition}
	For all agents $i \in \{1,\cdots,n\}$, we have $\sum_{j \geq 1}\mathbb{P}[O_i^{(j)} \neq 1 \vert \textrm{Good}] \leq K^{-2}$.
	\label{prop:late_good}
\end{proposition}
\begin{proof}
	Conditional on the event $\textrm{Good}$, we know that the best arm is played by all agents in the late phase. For any agent $i \in [1,n]$  and any phase $j \geq 1$, we can bound the error probability as 
	\begin{align*}
	\mathbb{P}[O_i^{(j)} \neq 1| \textrm{Good}] &\leq \frac{K}{\alpha-1} \left( \frac{1}{K} \lfloor \frac{T_0}{2} 2^{2^{j-1}} \rfloor - 1 \right)^{2(1-\alpha)} , \\
	&\leq \frac{K}{\alpha-1} (K2^{2^{j-1}-1}-1)^{2(1-\alpha)}, \\
	&\leq  \frac{K}{4(\alpha-1)} K^{2(1-\alpha)} 2^{2^{j}(1-\alpha)}.
	\end{align*} 
	The second inequality above follows from the fact that $T_0 \geq K^2$. By setting $\alpha = 3$, we get that $\mathbb{P}[O_i^{(j)}\neq 1| \textrm{Good}] \leq \frac{K^{-2}}{8} 2^{-2^{j+1}}$. The result follows from a simple series bound. 
\end{proof}
As a consequence of the above proposition, we obtain the following result.
\begin{proposition}
	For any $j \geq 0$, we have $\mathbb{E}[|A_I^{(j)}|] \leq 2M + \frac{3K}{n^2} + \lceil \frac{K}{n} \rceil$.
	\label{prop:arm_count}
\end{proposition}
\begin{proof}
	We have the basic decomposition.
	\begin{align*}
	\mathbb{E}[|A_I^{(j)}] &= \mathbb{E}[|A_I^{(j)} | \textrm{Good} ] \mathbb{P}[\textrm{Good}] +  \mathbb{E}[|A_I^{(j)} | \lnot \textrm{Good} ] \mathbb{P}[\lnot\textrm{Good}] , \\
	&\leq \mathbb{E}[|A_I^{(j)} | \textrm{Good} ]  + K (3n^{-2}),
	\end{align*}
	where in the second step we use the bound $ \mathbb{P}[\textrm{Good}]  \leq 1$ and $\mathbb{E}[|A_I^{(j)} | \lnot \textrm{Good} ] \leq K$ and the result of Lemma \ref{lem:early_good} to bound $ \mathbb{P}[\lnot\textrm{Good}] \leq 3n^{-2}$. Thus it remains to compute $\mathbb{E}[|A_I^{(j)} | \textrm{Good} ] \leq 2M+1$ to complete the proof.
	\\
	
	At the beginning of any phase, $|A_I^{(j)}|$ is $\lceil \frac{K}{n} \rceil$ (the initial number of arms per agent) plus the sum of distinct arm ids received by agent $I$ uptill the end of phase $j-1$. Conditional on the event $\textrm{Good}$, we know that agent $I$ will receive no more than $2M-1$ arms from all other agents, when the other agents were in phase $-1$. Furthermore, conditional on the event $\textrm{Good}$, all agents will have the best arm when they move to phase $0$. It thus remains to compute the expected number of arms received by agent $I$, when the agent recommending the arm is in a phase larger than or equal to $1$. From Proposition \ref{prop:late_good}, we know that with probability at-least $1-K^{-1}$, no agent will recommend an arm different from the best arm in any late-phase. This then gives by a total probability argument that 
	\begin{align*}
	\mathbb{E}[|A_I^{(j)}| \vert \textrm{Good}] \leq 2M -1 + \lceil \frac{K}{n} \rceil + K^{-1} K,
	\end{align*}
	where we assume the trivial upper bound of $K$, in the case that any agent in the late phase recommends an arm different from the best arm.
\end{proof}

Equipped with the above set of results, we are now ready to prove Theorem \ref{thm:known_delta}, on the regret experienced by agent $I$.  

\begin{proof}  

The regret of agent $I$ after $T$ epochs can be decomposed into three terms - 
\begin{itemize}
	\item The regret of at-most $ML$, for the $ML$ epochs in the early stage of agent $I$.
	\item The regret due to UCB algorithm in the late-stage of an agent. Here the number of arms played by agent $I$ in different late stage phases is different and random.
	\item An additional regret, if any paid until agent $I$ is aware of the best arm in the late-stage. 
\end{itemize}

The total regret, by linearity of expectation, is at-most the sum of the above three regret terms.
\\

	\textbf{Term $1$:}  	All agents pay a regret no larger than $ML$ in their early phase.
	\\

	\textbf{Term $2$:}  To do so, we need some notation. Denote by a sequence $(G_i)_{i \geq 0}$, where $G_0 = 0$ and $G_i = \lfloor \frac{T_0}{2}2^{2^{i-1}} \rfloor$, for $i \geq 1$. Notice that any agent plays for $G_{i+1} - G_{i}$ durations in phase numbered $i$. For any $T \in \mathbb{N}$, denote by $L_T \in \{0,..\}$ to be the last full phase played by agent $I$, i.e., $L_T := \max\{i \geq 0: G_i \leq T \}$. It is immediate to observe that $L_T \leq \log_2 \left(\log_2 \left(\frac{2T}{T_0}\right)\right)$. We will thus bound the regret as the sum of regret experienced by agent $I$ in the first $L_T+1$ phases of the late-stage. 
	
	\begin{align*}
\mathbb{E}[R_T^{(I);\text{Late-Stage}}] &\leq  \sum_{i=0}^{L_T+1} \mathbb{E}[R_{G_{i+1} - G_i; |A_I(T_i^{(I)})|}], \\
&\leq  \sum_{i=0}^{L_T+1} \mathbb{E}[\mathbb{E}[R_{G_{i+1} - G_i; |A_I(T_i^{(I)})|} | |A_I(T_i^{(I)})|], \\
&\stackrel{(a)}{\leq} \sum_{i=0}^{L_T+1} \mathbb{E} \left[\frac{4 \alpha}{\Delta}|A_I(T_i^{(I)})| \log(G_{i+1} - G_i) + |A_I(T_i^{(I)})| \left( 1 + \frac{\pi^2}{3} \right)\right],\\
&\stackrel{(b)}{\leq}  \sum_{i=0}^{L_T+1} \left( \frac{4 \alpha}{\Delta} \mathbb{E}[|A_I^{(j)}|] \log(G_{i+1}) + \mathbb{E}[|A_I^{(j)}|]\left( 1 + \frac{\pi^2}{3} \right) \right), \\
&\leq   \sum_{i=0}^{L_T+1}  \left(  \frac{4 \alpha}{\Delta} \mathbb{E}[|A_I^{(j)}|] \log(T_0 2^{2^i-1} )+   \mathbb{E}[|A_I^{(j)}|]\left( 1 + \frac{\pi^2}{3} \right)  \right), \\
&\leq   (L_T+2) \left(\frac{4 \alpha}{\Delta} \mathbb{E}[|A_I^{(j)}|] \log(T_0) +  \mathbb{E}[|A_I^{(j)}|]\left( 1 + \frac{\pi^2}{3} \right) \right) + \sum_{i=0}^{L_T + 1} \frac{4 \alpha}{\Delta} \mathbb{E}[|A_I^{(j)}|] \log(2)2^i, \\
&\stackrel{(d)}{\leq}   \frac{16 \alpha}{\Delta} \mathbb{E}[|A_I^{(j)}|]\log(T) +   \\&2\log_2 \left(\log_2 \left(\frac{2T}{T_0}\right)\right) \left( \frac{4 \alpha}{\Delta} \mathbb{E}[|A_I^{(j)}|] \log(T_0) + \mathbb{E}[|A_I^{(j)}|]\left( 1 + \frac{\pi^2}{3} \right)\right).
\end{align*}
		Inequality $(a)$ follows from the
	classical result on UCB($\alpha$) \cite{ucb_auer} and the fact that for all $i \geq 0$, $|A_I^{(j)}|$ is independent of the regret incurred by agent $I$ in state $j$. Inequality $(b)$ follows from replacing $G_{i+1} - G_i \leq G_{i+1}$ and Proposition \ref{prop:arm_count}. Inequality $(d)$ follows from the fact that $L_T + 2 \leq 2\log_2 \left(\log_2 \left(\frac{T}{T_0}\right)\right)$. Recall that an upper bound for $\mathbb{E}[|A_I^{(j)}|]$ is given in Proposition \ref{prop:arm_count}. 
	\\
	
	\textbf{Term $3$: }
	If the event Good holds, then all agents are aware of the best arm at the beginning of their late-phase and hence do not pay any additional regret apart for terms $1$ and $2$. In the rare case that the Good event does not hold, which from Lemma \ref{lem:early_good}, we know happens with probability at-most $150n^{-3}$, we know from Proposition \ref{prop:S_Value_integer} that on average, agent $I$ does not play the best arm until the end of phase $2$. Thus, conditional on the system not being Good, the additional regret played by agent $I$ is at-most the number of epochs it takes to move from the beginning of phase $0$ to phase $S+1$, which on average is $8T_0$. Since this occurs with probability at-most $150 \log(n) n^{-3}$, the regret accounted for the third term is at-most $150\log(n)n^{-3} 8T_0= \frac{1200 \log(n)\log(\varepsilon^{-1}) \max(K^2,n)}{n^3\varepsilon^2}$.

\end{proof}

%% file: proof_rumor.tex
\section{Proof of Theorem \ref{thm:rumor_spread}}
\label{sec:proof_rumor}

In order to prove Theorem \ref{thm:rumor_spread}, we shall consider a noisy version of the classical rumor spreading process of  \cite{discrete_rumor} and \cite{pittel}. Suppose there are $n$ agents with agent $1$ holding a message at time $0$. For each agent $i \in [n]$, denote by time $\tilde{Y}_i \in \mathbb{N}$, to be the first time when agent $i$ is aware of the message. By definition, we have that $\tilde{Y}_1 = 0$. In each time step $t \in \mathbb{N}$, every agent $i$ such that $Y_i \leq t-1$, will attempt to communicate the message to another agent chosen uniformly and independently at random. Each communication successfully communicates the message with probability $p \in (0,1]$. Denote by $\tilde{S}_n^{(p)} = \max_{i \in [n]} \tilde{Y}_i$, to be the first time when all agents are aware of the message. The following theorem sheds light on the growth of the random variable $\tilde{S}_n^{(p)} $.

\begin{theorem}
	Let $p \in (0,1]$ and $\gamma >0$ be arbitrary and $V = \frac{1}{\log \left( \frac{1}{\frac{7}{10}p+\bar{p}}\right)} (\gamma+1)$. For all $n$ sufficiently large such that $pV \frac{\log(n)}{n}   + \bar{p} < 1$, we have
	\begin{align}
	\mathbb{P}[\tilde{S}_n^{(p)} \geq {C}(\gamma,p)\log(n)] \leq (2+ \log_{2-\eta}(n))n^{-(\gamma+1)},
	\label{eqn:rumor_spread_to_establish}
	\end{align} 
	where ${C}(\gamma,p) =  \log_{2-\eta}(n) + \left(D + \frac{3}{p}(3 + 2\gamma)\right)\log(n)$. The constant $\eta \in  \left( \frac{2}{3}p + \bar{p},1\right)$ is the smallest possible number such that 
	\begin{align}
	\left( \frac{\frac{2p}{3}  + \bar{p}}{\eta}\right)^{\eta} \frac{1}{(1-\eta)^{(1-\eta)}} < \frac{7}{10}p+\bar{p} < 1,
	\label{eqn:zeta_defn}
	\end{align}
	and $D$ is the smallest positive number so that the following equation is satisfied
	\begin{align}
 \inf_{A > 1} \left( A pV \frac{\log(n)}{n}   + A\bar{p}\right)^{D}  \left( \frac{Ap}{(A-1)\left(A pV \frac{\log(n)}{n}   + A\bar{p}\right)}  \right)^V  \leq e^{-(\gamma+1)}.
	\label{eqn:D_defn}
	\end{align}
	\label{thm:discrete_rumor_2}
\end{theorem}


A choice of $\eta$ exists since $\lim_{\eta \nearrow 1} \left( \frac{\zeta 2p + \bar{p}}{\eta}\right)^{\eta} \frac{1}{(1-\eta)^{(1-\eta)}} =  \zeta 2p + \bar{p} < 1$, for all $\zeta \in (0,1/2)$. A choice of $A > 1$ exists since $pV \frac{\log(n)}{n}   + \bar{p} < 1$. 


\begin{remark}
	For the case of $p=0.99$ and $\gamma = 2$, we have $V = 8.52$ and $\eta = 0.993$. For all $n \geq 29$, we have $pV \frac{\log(n)}{n}   + \bar{p} < 1$. This gives us $D = 15.85$. This gives us that $C(2,0.99) \leq 180.413$.
		\label{remark:constant_C}
\end{remark}
%

Before we give the proof, we notice that Theorem \ref{thm:discrete_rumor_2} immediately yields Theorem \ref{thm:rumor_spread} as a corollary. For any $x \in \mathbb{N}$, we have the following stochastic domination 
\begin{align}
\mathbb{P}[S_n^{(p)} \geq 2x] \leq \mathbb{P}[\tilde{S}_n^{(p)} \geq x].
\label{eqn:spread_stoch_dom}
\end{align}
This above Equation follows as one can view an upper bound to the delayed process where agents only call in even numbered time-slots. The time for everyone to know the message in this call only at even time slots process is clearly lower bounded by the time taken for all agents to know the message when only the newly informed agents keep silent instead of all agents, i.e., the one step delayed process. Hence, in light of Equation (\ref{eqn:spread_stoch_dom}), it suffices to establish Theorem \ref{thm:discrete_rumor_2} in order to prove Theorem \ref{thm:rumor_spread}.

%

\begin{proof}
The proof follows similar ideas used in \cite{discrete_rumor} and \cite{pittel}. We set some notations to carry out the proof. Denote by the `state' of the system at time $t \in \mathbb{N}$ to be the number of agents that are aware of the rumor. For simplicity, at any time $t \in \mathbb{N}$ and state $i \in [n]$, we will assume that agent $1$ makes $i$ different calls in this time step. Each call of agent $1$ communicates the message across with probability $p$ independent of everything else. For $i \in [n-1]$, denote by $W_{i}$ as the number of calls needed to be made by agent $1$ to move the system from state $i$ to $i+1$. Let $\bar{p} := 1-p$. Clearly, the following holds true.
\begin{align*}
\mathbb{P}[W_i = r] &= \left( \frac{i}{n}p + \bar{p} \right)^{r-1} \left( 1 - \frac{i}{n} \right)p \\
\mathbb{E}[e^{tW_i}] &= \frac{p(n-i)}{e^{-t}n - (ip +n\bar{p})}, \text{  } \forall t\geq 0, \text{ s.t. } e^t < \frac{n}{ip+n\bar{p}}.
\end{align*} 

It is immediate to verify that for all $i < j$ and all $t < \frac{n}{jp+n\bar{p}}$, we have 
\begin{align}
\mathbb{E}[e^{tW_i}] \leq \mathbb{E}[e^{tW_j}].
\label{eqn:comparision_mgf}
\end{align}

In order to establish Equation (\ref{eqn:rumor_spread_to_establish}), we shall consider the spread of rumor in phases as done in \cite{discrete_rumor}. Recall from the theorem statement that the constant $V = \frac{1}{\log \left( \frac{1}{\frac{7}{10}p+\bar{p}}\right)} (\gamma+1)$. We will establish the following.


\begin{enumerate}
	
	\item With probability at-least  $1- n^{-(\gamma +1)}$, the number of informed agents increases from $1$ to $V \log(n)$, in at-most $D\log(n)$ time where $V$ is in Theorem \ref{thm:discrete_rumor_2} and $D$ is given in Equation (\ref{eqn:D_defn}).
	\item With probability at-least $1 - \log_{2-\eta}(n)n^{-(\gamma +1)}$, where $\eta$ is given in Equation (\ref{eqn:zeta_defn}), the number of informed agents increases from $V \log(n)$ to $n/3$ in at-most $\log_{2-\eta}(n)$ time.
	\item With probability at-least $1-n^{-(\gamma +1)}$, the number of informed agents increases from $n/3$ to $n$ in at-most $\frac{3}{p}(3+2\gamma) \log(n)$ time.
	
%
\end{enumerate}
If the above statements hold true, then the theorem is concluded by a straightforward union bound.
In what follows we establish each of the above three claims separately.
\\

\textbf{Step $1$} - \\

 The probability that it takes more than $D \log(n)$ time to inform $V\log(n)$ people is upper bounded by $\mathbb{P}[W_1 + \cdots + W_{V\log(n)} \geq D \log(n)]$ which in turn can be upper-bounded as follows.
\begin{align*}
&\mathbb{P}[W_1 + \cdots + W_{V\log(n)} \geq D \log(n)]\\&\stackrel{(a)}{\leq} e^{-D \log(n)t} \prod_{i=1}^{V\log(n)}  \frac{p(n-i)}{e^{-t}n - (ip +n\bar{p})}, \\
&\stackrel{(b)}{\leq} e^{-D \log(n)t} \left(  \frac{pn}{e^{-t}n - (V\log(n)p +n\bar{p})} \right)^{V\log(n)}, \\
&\stackrel{(c)}{\leq} \left(\left( Vp \frac{\log(n)}{n} + \bar{p}  \right)A \right)^{(D-V)\log(n)}\left( \left(\frac{Ap}{A-1} \right) \right)^{V \log(n)}.
\end{align*}
In step $(a)$, we use the classical Chernoff type bound where for a positive random variable $X$, for all $x \geq 0$, $\mathbb{P}[X \geq x] \leq e^{-tx}\mathbb{E}[e^{tX}]$, for all $t \geq 0$. In step $(b)$, we use Equation (\ref{eqn:comparision_mgf}). In step $(c)$, we use $e^t = \frac{n}{A(\log(n)p + n \bar{p})}$ for an appropriate value of $A > 1$. Thus, we get
\begin{multline*}
\mathbb{P}[W_1 + \cdots + W_{V\log(n)} \geq D \log(n)] \\ \leq \left(  \left( A pV \frac{\log(n)}{n}   + A\bar{p}\right)^{D}  \left( \frac{Ap}{(A-1)\left(A pV \frac{\log(n)}{n}   + A\bar{p}\right)}  \right)^V  \right)^{\log(n)}. 
\end{multline*}
From the choice of the constant $D$ specified in Equation (\ref{eqn:D_defn}) and $n$ is sufficiently large as specified in the theorem, we can choose $A >1$ such that $\mathbb{P}[W_1 + \cdots + W_{V\log(n)} \geq D \log(n)]  \leq n^{-(\gamma+1)}$.
\\

\textbf{Step $2$} - \\

We will show that with high probability, starting from $V \log(n)$ informed agents, in each time step, the number of informed agents multiples by a factor of at-least $(2 - \eta) > 1$ until $ n/3$ agents are informed for the first time. This implies that, with probability at-least $n^{-(\gamma+1)}$, in at-most $\log_{2-\eta}(n)$ steps, the total number of informed agents rise from $V\log(n)$ to $n/3$. 
\\

More precisely, we will argue that if at some time $t$, the number of informed agents is $V \log(n)$, then, with high probability, for all $k \in \mathbb{N}$ such that $V \log(n)(2-\eta)^k \leq n/3$, the number of informed agents at time $t+k$ is at-least $V \log(n)(2-\eta)^k$. Thus, within $ \phi := \log_{2-\eta}(n/3V\log(n))$, steps, the number of informed people increase from $V \log(n)$ to $n/3$ with high probability. To implement this proof, define recursively, the following events. Event $D(1)$ states that given there are at-least $V \log(n)$ informed agents, there are lesser than $V \log(n)(2-\eta)$ agents in the next time step. The event $D(i)$, for $i \in \{1,\phi\}$ is the event that starting from at-least $V\log(n)(2-\eta)^{i-1}$ informed agents, there are fewer than $V \log(n)(2-\eta)^i$ informed agents in the next time step. To argue about Step $2$, it suffices to bound $\sum_{i=1}^{k} \mathbb{P}[D(i)]$. 
For any $i \in \{1,\cdots,\phi\}$, we have 
\begin{align*}
\mathbb{P}[D(i)] \leq \mathbb{P}[W_j..+W_{j(2-\eta)} > i ],
\end{align*}
where $j = V \log(n)(2-\eta)^{i-1}$. Following the steps we outlined earlier, i.e., the Chernoff bound and Equation (\ref{eqn:comparision_mgf}), we bound this probability as 
\begin{align*}
\mathbb{P}[W_j..+W_{j(2-\eta)} > j]& \leq e^{-jt} \prod_{l=j}^{j(2-\eta)} \mathbb{E}[e^{tW_l}],\\
&\leq e^{-jt} \left( \frac{p(n-j(2-\eta))}{e^{-t}n - (j(2-\eta)p + n \bar{p})} \right)^{j(1-\eta)}, \\
&\leq e^{-jt \eta} \left( \frac{pn}{n - e^{t}(jp(2-\eta) + n\bar{p})}\right)^{j(1-\eta)},\\
&\leq \left(\left( \frac{j}{n} \frac{p(2-\eta)}{\eta} + \frac{\bar{p}}{\eta} \right)^{\eta} \left( \frac{p}{1-\eta} \right)^{(1-\eta)}\right)^j
\end{align*}
where in the last step, we substitute $t$ such that $e^t = \frac{n \eta}{jp(2-\eta)+n\bar{p}}$. Thus, we have 
\begin{align*}
\mathbb{P}[\cup_{i=1}^{\phi}D(i)] &\leq \sum_{i=1}^{\phi} \mathbb{P}[D(i)], \\
&\leq \sum_{i=1}^{\phi} \left(\left( \frac{j}{n} \frac{p(2-\eta)}{\eta} + \frac{\bar{p}}{\eta} \right)^{\eta} \left( \frac{p}{1-\eta} \right)^{(1-\eta)}\right)^{V \log(n)(2-\eta)^{i-1}}, \\
&\leq \phi \left(\left( \zeta \frac{p(2-\eta)}{\eta} + \frac{\bar{p}}{\eta} \right)^{\eta} \left( \frac{p}{1-\eta} \right)^{(1-\eta)}\right)^{V\log(n)}, \\
&\leq \log_{2-\eta}(n) \left(  \frac{7}{10}p+\bar{p} \right)^{V \log(n)} , \\
&\leq \log_{2-\eta}(n)n^{-(\gamma+1)}.
\end{align*}

\textbf{Step $3$} - \\

We once again employ the Chernoff bound and Equation (\ref{eqn:comparision_mgf}) to our benefit. We shall bound the probability of the total number of calls taken to move from one agent knowing the rumor to all $n$ agents knowing the rumor  is order $n\log(n)$ with probability at-least $1-n^{-(\gamma+1)}$. However, as we have at the beginning of this phase, at-least $n/3$ informed agents, this calculation will give that at-most order $\log(n)$ calls suffice to move the system where $n/3$ agents know the message to all $n$ agents knowing the message. Denote by $\Lambda = \frac{3+2\gamma}{p}$. The computation is as follows 
\begin{equation*}
\mathbb{P}[W_1 + \cdots + W_{n-1} \geq \Lambda n \log(n) ]  \leq \\ e^{-t \Lambda n \log(n)} \prod_{j=1}^{n-1} \frac{p(n-j)}{e^{-t}n - (jp + n \bar{p})}.
\end{equation*}
We will substitute $e^{-t} = 1 - p (2n)^{-1}$ in the above expression to obtain
\begin{align*}
\mathbb{P}[W_1 + \cdots + W_{n-1} \geq \Lambda n \log(n) ] &\leq\left( 1 - \frac{Rp}{2n} \right)^{\Lambda n \log(n)} \prod_{j=1}^{n-1} \frac{p(n-j)}{n - \frac{p}{2} - (jp + n \bar{p})}, \\
&= \left( 1 - \frac{p}{2n} \right)^{\Lambda n \log(n)} \prod_{j=1}^{n-1} \frac{p(n-j)}{   pn - \frac{p}{2} - jp  }, \\
&= \left( 1 - \frac{p}{2n} \right)^{\Lambda n \log(n)} \prod_{j=1}^{n-1} \frac{(n-j)}{   n - \frac{1}{2} - j }, \\
&= \left( 1 - \frac{p}{2n} \right)^{\Lambda n \log(n)} \prod_{j=1}^{n-1} \frac{2j}{2j - 1}, \\
&= \left( 1 - \frac{p}{2n} \right)^{\Lambda n \log(n)} \prod_{j=1}^{n-1} \left( 1 + \frac{1}{2j-1} \right), \\
&\leq  \left( 1 - \frac{p}{2n} \right)^{\Lambda n \log(n)} \exp \left(  \sum_{j=1}^{n-1} \frac{1}{2j-1}\right), \\
&\leq \left( 1 - \frac{p}{2n} \right)^{\Lambda n \log(n)} \exp \left(  \int_{x=1}^{n} \frac{1}{2x}dx \right), \\
&\leq \left(\left( 1 - \frac{p}{2n} \right)^{\Lambda n} \exp \left(\frac{1}{2} \right) \right)^{\log(n)}, \\
& \leq e^{-(\gamma+1) \log(n)}, \\
& \leq  n^{-(\gamma+1)}.
\end{align*}
 since $\Lambda = \frac{3+2\gamma}{p}$. Thus, a total of $\Lambda n\log(n)$ calls from agent $1$ suffices to inform all agents. However, as there are at-least $n/3$ informed agents in each round, the above procedure takes at-most $3\Lambda \log(n)$ time steps to complete.

\end{proof}